%% file: main.tex
\documentclass{article}

\usepackage[preprint]{neurips_2026}

\usepackage[utf8]{inputenc} 
\usepackage[T1]{fontenc}    
\usepackage{hyperref}       
\usepackage{url}            
\usepackage{booktabs}       
\usepackage{amsfonts}       
\usepackage{nicefrac}       
\usepackage{microtype}      
\usepackage{xcolor}         
\usepackage{amsthm}
\usepackage{dsfont}
\usepackage{thmtools}
\usepackage{thm-restate}
\usepackage{subcaption}
\usepackage{authblk}
\input{macros}
\newtheorem{theorem}{Theorem}[section]
\newtheorem{proposition}[theorem]{Proposition}
\newtheorem{lemma}[theorem]{Lemma}

\newtheorem{definition}[theorem]{Definition}

\title{Near-Optimal Last-Iterate Convergence for Zero-Sum Games with Bandit Feedback and Opponent Actions}

\author[1]{Soumita Hait}
\author[2]{Ping Li}
\author[1]{Haipeng Luo}
\author[3]{Mengxiao Zhang}

\affil[1]{University of Southern California, \texttt{\{hait,haipengl\}@usc.edu}}
\affil[2]{Shanghai University of Finance and Economics, \texttt{pinglee@stu.sufe.edu.cn}}
\affil[3]{University of Iowa, \texttt{mengxiao-zhang@uiowa.edu}}

\begin{document}

\maketitle

\begin{abstract}
Last-iterate convergence of learning dynamics in games has attracted significant recent attention. In two-player zero-sum games with bandit feedback, where only the loss of the selected action pair is observed, \citet{fiegel2025harder} show a separation between average-iterate and last-iterate convergence in duality gap: while the optimal $t^{-\frac{1}{2}}$ rate after $t$ rounds is achievable for the former via standard no-regret algorithms, the latter cannot converge faster than $t^{-\frac{1}{3}}$ in expectation or $t^{-\frac{1}{4}}$ with high probability. However, in many practical settings (such as preference learning), the players observe not only their loss but also the opponent’s action. This raises a natural question: can such additional information enable faster last-iterate convergence? 
We answer this question affirmatively, showing that $t^{-\frac{1}{2}}$ last-iterate convergence is achievable with high probability in this setting, via an efficient algorithm that updates its strategy infrequently by solving an estimated log-barrier-regularized game. We identify fundamental obstacles preventing standard analysis for multi-armed bandits (the single-player case) from generalizing to games, and develop a novel analysis to overcome them. Experiments confirm that our algorithm indeed converges faster than naive baselines and prior methods that do not exploit opponent-action feedback. Finally, we note that our results also improve those for dueling bandits, a special case with skew-symmetric game matrices. 

\end{abstract}

\input{sections/introduction}

\input{sections/preliminaries}

\input{sections/challenges}
\input{sections/algorithm}

\input{sections/experiments}

\bibliographystyle{plainnat}
\bibliography{references}


\appendix

\input{appendix/omitted_proofs}


\end{document}

%% file: macros.tex
\newif\ifspacehack
\usepackage{natbib}
\usepackage{hyperref}
\hypersetup{
    colorlinks = true,
    breaklinks,
    linkcolor = blue,
    citecolor = blue,
    urlcolor  = blue,
}
\usepackage{url} 
\usepackage{graphicx}
\usepackage{mathtools}
\usepackage{footnote}
\usepackage{float}
\usepackage{xspace}
\usepackage{multirow}
\usepackage{xcolor}
\usepackage{wrapfig}
\usepackage{framed}
\usepackage{bbm}
\usepackage[most]{tcolorbox}

\usepackage{footnote}
\usepackage{nicefrac}
\usepackage{makecell}

\usepackage[ruled,vlined]{algorithm2e}
\usepackage{amssymb}
\usepackage{bm}
\makesavenoteenv{tabular}
\makesavenoteenv{table}

\newcommand\innerp[2]{\langle #1, #2 \rangle}
\renewcommand{\tilde}{\widetilde}

\def \R {\mathbb{R}}

\newcommand{\calE}{{\mathcal{E}}}

\newcommand{\Reg}{{\mathrm{Reg}}}

\newcommand{\E}{{\mathbb{E}}}

\newcommand{\inner}[1]{ \left\langle {#1} \right\rangle }

\newcommand{\order}{\mathcal{O}}

\newcommand{\one}{\boldsymbol{1}}

\newcommand{\dgap}{\text{\rm DGap}}

\DeclareMathOperator*{\argmin}{argmin}
\DeclareMathOperator*{\argmax}{argmax}

\newcommand{\wh}{\widehat}

\newcommand{\otil}{\ensuremath{\tilde{\mathcal{O}}}}

\newcommand{\rbr}[1]{\left(#1\right)}
\newcommand{\sbr}[1]{\left[#1\right]}
\newcommand{\cbr}[1]{\left\{#1\right\}}

\usepackage{lipsum,booktabs}
\usepackage{amsmath,mathrsfs,amssymb,amsfonts,bm,enumitem}
\usepackage{rotating}
\usepackage{pdflscape}
\usepackage{hyperref,url}
\hypersetup{
    colorlinks,
    breaklinks,
    linkcolor = blue,
    citecolor = blue,
    urlcolor  = blue,
}
\allowdisplaybreaks
\usepackage{appendix}
\usepackage{multirow,makecell}

\renewcommand{\tilde}{\widetilde}

\def \E {\mathbb{E}}

\def \R {\mathbb{R}}

\def \Pr {\mathsf{Pr}}

\usepackage{mathtools}

\DeclarePairedDelimiter\abs{\lvert}{\rvert}

\newtcolorbox{promptbox}[1][]{
  colback=blue!5!white, colframe=blue!75!black,
  fonttitle=\bfseries, title=Prompt,
  left=2mm, right=2mm, top=2mm, bottom=2mm,
  boxrule=0.5mm,  
  coltitle=black, 
  colbacktitle=blue!15!white, 
  breakable,      
  #1
}

\usepackage{graphicx,color} 

\definecolor{wine_red}{RGB}{228,48,64}
\definecolor{DSgray}{cmyk}{0,1,0,0}

\usepackage{prettyref}
\newcommand{\pref}[1]{\prettyref{#1}}

\newcommand{\savehyperref}[2]{\texorpdfstring{\hyperref[#1]{#2}}{#2}}
\newrefformat{eq}{\savehyperref{#1}{Eq.~\textup{(\ref*{#1})}}}
\newrefformat{eqn}{\savehyperref{#1}{Eq.~(\ref*{#1})}}
\newrefformat{lem}{\savehyperref{#1}{Lemma~\ref*{#1}}}
\newrefformat{event}{\savehyperref{#1}{Event~\ref*{#1}}}
\newrefformat{def}{\savehyperref{#1}{Definition~\ref*{#1}}}
\newrefformat{line}{\savehyperref{#1}{Line~\ref*{#1}}}
\newrefformat{thm}{\savehyperref{#1}{Theorem~\ref*{#1}}}
\newrefformat{tab}{\savehyperref{#1}{Table~\ref*{#1}}}
\newrefformat{corr}{\savehyperref{#1}{Corollary~\ref*{#1}}}
\newrefformat{cor}{\savehyperref{#1}{Corollary~\ref*{#1}}}
\newrefformat{sec}{\savehyperref{#1}{Section~\ref*{#1}}}
\newrefformat{app}{\savehyperref{#1}{Appendix~\ref*{#1}}}
\newrefformat{assum}{\savehyperref{#1}{Assumption~\ref*{#1}}}
\newrefformat{asm}{\savehyperref{#1}{Assumption~\ref*{#1}}}
\newrefformat{asp}{\savehyperref{#1}{Assumption~\ref*{#1}}}
\newrefformat{ex}{\savehyperref{#1}{Example~\ref*{#1}}}
\newrefformat{fig}{\savehyperref{#1}{Figure~\ref*{#1}}}
\newrefformat{alg}{\savehyperref{#1}{Algorithm~\ref*{#1}}}
\newrefformat{rem}{\savehyperref{#1}{Remark~\ref*{#1}}}
\newrefformat{conj}{\savehyperref{#1}{Conjecture~\ref*{#1}}}
\newrefformat{prop}{\savehyperref{#1}{Proposition~\ref*{#1}}}
\newrefformat{proto}{\savehyperref{#1}{Protocol~\ref*{#1}}}
\newrefformat{prob}{\savehyperref{#1}{Problem~\ref*{#1}}}
\newrefformat{claim}{\savehyperref{#1}{Claim~\ref*{#1}}}
\newrefformat{que}{\savehyperref{#1}{Question~\ref*{#1}}}
\newrefformat{op}{\savehyperref{#1}{Open Problem~\ref*{#1}}}
\newrefformat{fn}{\savehyperref{#1}{Footnote~\ref*{#1}}}

\def \epsilon {\varepsilon}



%% file: sections/introduction.tex
\section{Introduction}\label{sec:intro}
Understanding learning dynamics in games is an important area in game theory and machine learning. In particular, last-iterate convergence to Nash equilibria of two-player zero-sum games has received increasing attention since it does not require averaging iterates (which could be cumbersome for neural nets) and provides strong day-to-day behavior guarantees. There has been a long line of research on this topic, much of which focuses on an idealized setting where each player obtains gradient feedback of their strategies. In this case, it has been shown that $t^{-1}$ last-iterate convergence in duality gap is achievable after playing the games for $t$ rounds; see e.g.~\citet{diakonikolas2020halpern, yoon2021accelerated, cai2023doubly, cai_accelerated_2023, cai2024accelerated, cai2025average}.

However, gradient feedback might be too much to ask for in many applications. Often, players can see a (noisy) loss/reward of their selected action, but not the gradient of their mixed strategy. This is referred to as the bandit feedback setting and can be viewed as a direct generalization of stochastic multi-armed bandits (that is, single-player games) to multi-player games. Under this more challenging model, even for the weaker notion of average-iterate convergence, the best rate one can hope for is $t^{-\frac{1}{2}}$ due to standard statistical barriers, and any bandit algorithms with standard $\sqrt{t}$-regret achieve such average-iterate convergence~\citep{freund1999adaptive}.

Achieving last-iterate convergence with bandit feedback, however, is even more challenging due to the severer tension between exploration and exploitation. Earlier results establish slow rates such as $t^{-\frac{1}{8}}$~\citep{cai2023uncoupled, chen2023finite, chen2024last}, $t^{-\frac{1}{6}}$~\citep{cai2023uncoupled}, and $t^{-\frac{1}{5}}$\citep{cai2025average}. More recently, \citet{fiegel2025harder, fiegel2026optimal} show that with bandit feedback, there is indeed a strict separation between average-iterate convergence, for which the optimal rate is $t^{-\frac{1}{2}}$, and last-iterate convergence, for which the optimal in-expectation rate is $t^{-\frac{1}{3}}$ and the optimal high-probability rate is $t^{-\frac{1}{4}}$. 

While bandit feedback is more realistic than gradient feedback, in many situations, in addition to the losses, the players also observe their opponents' actions. For example, in pricing, companies can naturally observe their competitors' price; in online advertising auctions, bidders can observe the winning bid; and in security games, a defender may observe which targets an attacker chose to probe. There are also games where a central controller decides the actions for all players, in which case the action information is also directly available --- a salient example is dueling bandits~\citep{yue2012k} or more generally learning from human preference feedback, often modeled as learning over a skew-symmetric preference game matrix where the algorithm decides which pair of actions to recommend to the user; see e.g.,~\citet{munos2024nash, swamy2024minimaximalist}. Such additional feedback intuitively facilitates estimation and mitigates the exploration burden, leading to a central question that our work focuses on: \textit{does observing the opponent's actions enable faster last-iterate convergence in two-player zero-sum games?}

\paragraph{Contribution}
We answer this question affirmatively and develop an algorithm with $t^{-\frac{1}{2}}$ high-probability last-iterate convergence, which, up to dimension dependence, is optimal and closes the gap between average-iterate and last-iterate convergence.
Our algorithm is computationally efficient and only updates its strategy $\log t$ times after $t$ rounds, where each update solves an estimated log-barrier-regularized game.

Our algorithm shares similarity with the FALCON algorithm of~\citet{simchi2022bypassing} for contextual bandits and can be viewed as a generalization of their algorithm to two-player zero-sum games (but specialized to the non-contextual setting).
However, we highlight that our analysis significantly deviates from theirs and is the main technical contribution of our work.
Indeed, we provide extensive discussion in \pref{sec:obstacle} on why even with the extra information of opponent actions, achieving $t^{-\frac{1}{2}}$ last-iterate convergence is still highly challenging and existing analysis in the bandit literature does not generalize to games.
This is true even when the game has a pure-strategy Nash equilibrium.
At a high level, the key obstacle is that estimating the duality gap of a strategy requires playing other (potentially bad) strategies.
Nevertheless, we overcome this obstacle via a novel analysis that relies on proving multiplicative stability between two consecutive strategies and might be of independent interest even for the single-player case.
We also highlight that since dueling bandits are a special case of our setting, our results also strengthen this line of work by providing last-iterate convergence guarantees rather than only regret bounds. 

Finally, we conduct experiments on the security game benchmark of~\citet{krever2025guard}, showing that our algorithm consistently outperforms both existing methods that ignore opponent-action feedback and naive baselines that exploit it.

\paragraph{Related work}
There is a vast literature on learning in games. We refer the readers to the related work section of~\citet{cai2025average} for prior work under the gradient feedback model and only focus on related work under the bandit feedback model here.

\paragraph{Regret and average/random-iterate convergence}
The seminal result of~\citet{freund1999adaptive} shows that average-iterate convergence rate in two-player zero-sum games is controlled by the social regret of the players.
In particular, standard multi-armed bandit algorithms~\citep{auer2002nonstochastic, zimmert2021tsallis} with optimal $\sqrt{t}$-regret ensure $t^{-\frac{1}{2}}$ average-iterate convergence.
To go beyond the $t^{-\frac{1}{2}}$ barrier, recent work either considers weaker regret notions or develop instance-dependent regret bounds~\citep{maiti2025limitations, ito2025instance, ito2026adversarial}.
\citet{sessa2019no, o2021matrix, ito2026adversarial} also study the effect of opponent-action information on regret, and
we refer to~\citet{ito2026adversarial} for more discussion related to the line of research on Markov games where the opponent-action information also plays an important role.

On the other hand, social dynamic regret, which is equivalent to the cumulative duality gap, controls 
the random-iterate convergence of the learning dynamic. 
The algorithms of \citet{xie2020learning, bai2020provable} based on the well-known optimism principle enjoy $\sqrt{t}$ social dynamic regret (with opponent-action information) and thus $t^{-\frac{1}{2}}$ random-iterate convergence.
Note that our $t^{-\frac{1}{2}}$ last-iterate convergence trivially implies the same for random-iterate convergence, but not the other way around.
In fact, it is known that such optimistic algorithms provably fail to achieve last-iterate convergence even for multi-armed bandits~\citep{liu2024uniform}.

\paragraph{Last-iterate convergence}
As mentioned, a line of recent work studies last-iterate convergence without opponent-action feedback, improving the results from $t^{-\frac{1}{8}}$ down to the optimal in-expectation rate $t^{-\frac{1}{3}}$ and the optimal high-probability rate $t^{-\frac{1}{4}}$~\citep{cai2023uncoupled, chen2023finite, chen2024last, cai2025average, fiegel2025harder, fiegel2026optimal}.
\citet{maiti2026efficient} achieve similar $t^{-\frac{1}{4}}$ rate but require seeing the loss of the mixed strategies instead of the loss of the selected actions.
Another recent work by~\citet{ito2025instance} establishes $t^{-\frac{1}{2}}$ last-iterate convergence for games with a unique pure strategy Nash equilibrium, but the bound involves game-dependent constants that can be arbitrarily large.
To our knowledge, no prior work has studied the benefits of opponent-action information for last-iterate convergence.

\paragraph{Connection to bandits}
When one player has only one feasible action, our problem becomes a single-player game and recovers the classic stochastic multi-armed bandit problem~\citep{lai1985asymptotically}.
Similarly, most work in this area focuses on developing regret guarantees.
To our knowledge, \citet{liu2024uniform} are the first to explicitly study last-iterate convergence, although an earlier result by~\citet{simchi2022bypassing} also achieves last-iterate convergence implicitly for the more general contextual bandit problem.
We will discuss in depth why these approaches do not generalize to game in \pref{sec:obstacle}.
Similar to~\citet{ito2025instance}, \citet{zhan2025last} develop last-iterate convergence rate (for multi-armed bandits) that involves potentially large instance-dependent constants.

%% file: sections/preliminaries.tex
\section{Preliminaries}\label{sec:preliminaries}
\paragraph{General notations} 
For a positive integer $k$, let $[k]$ denote the set $\{1,2,\dots,k\}$.
Define $\R^d_+$ to be the positive orthant of the $d$-dimensional Euclidean space, and $\Delta_d\triangleq\{x\in \R^d_+, \sum_{i=1}^d x_i=1\}$ to be the $(d-1)$-dimensional simplex. 
For a vector $v\in \R^{d}$, we denote by $v_i$ its $i$-th entry.
Similarly, for a matrix $M\in \R^{m\times n}$, we denote by $M_{ij}$ its $(i,j)$-th entry.
Let $e_i$ be the one-hot vector in an appropriate dimension with the $i$-th entry being $1$ and all other entries being $0$. 
Let $\mathbf{1}$ denote the all-one vector in an appropriate dimension.

\paragraph{Two-player zero-sum games} 
We study two-player zero-sum games specified by a game matrix $A\in[-1,1]^{d\times d}$, where $A_{ij}$ denotes the expected loss incurred by the row player and, equivalently, the expected reward obtained by the column player, when the row player selects action $i$ and the column player selects action $j$.\footnote{For the ease of presentation, we assume that both players have the same number of actions, but our results directly generalize to the case where they have different number of actions.} 
For a mixed strategy pair $(x,y) \in \Delta_d \times \Delta_d$, the expected loss of the row player is given by $x^\top A y$ (which is also the expected reward of the column player), and the duality gap of this strategy pair is defined as
\[
\dgap(x,y) \triangleq \max_{y'\in\Delta_d}x^\top Ay' - \min_{x'\in\Delta_d}x'^\top Ay = \max_{i,j\in [d]} x^\top A e_j - e_i^\top A y,
\]
which is always non-negative by definition.
It is well-known that a strategy pair $(x,y)$ has $0$ duality gap if and only if it is a Nash equilibrium of the game (that is, both players are best-responding to each other: $x^{\top}Ay' \le x^{\top}Ay \le {x'}^\top Ay$ for all $x', y'\in\Delta_d$).
Therefore, duality gap is a natural measure for convergence rate to Nash equilibria. 

While Nash equilibria always exist, 
a pure-strategy Nash equilibrium (PSNE) of the form $(e_i, e_j)$ for some $i,j \in[d]$ might not.
In \pref{sec:obstacle}, we will discuss why our problem is challenging even when the game has a unique PSNE, but we emphasize that our results hold generally, without any assumption on the games (other than boundedness).

\paragraph{Learning via repeated play with bandit feedback and opponent actions}
We consider a standard repeated play setting in which both players have no prior knowledge of the game matrix $A$. In each round $t = 1,2,\ldots$ , the row player selects a mixed strategy $x_t \in \Delta_d$ and samples an action $i_t \sim x_t$;
similarly, the column player selects a mixed strategy $y_t \in \Delta_d$ and samples an action $j_t \sim y_t$. Following that, both players observe bandit feedback, that is, a noisy loss/reward signal $r_t \in [-1,1]$ satisfying $\E\sbr{r_t}=A_{i_tj_t}$. 
In addition, row player observes the opponent's action $j_t$ and similarly the column player observes $i_t$.

An algorithm satisfies last-iterate convergence if it ensures $\dgap(x_t, y_t)\rightarrow 0$ as $t$ increases, which is stronger than other heavily-studied measures such as average-iterate convergence:  $\dgap\rbr{\frac{1}{t}\sum_{\tau=1}^t x_\tau, \frac{1}{t}\sum_{\tau=1}^t y_\tau} \rightarrow 0$,
random-iterate convergence: $\frac{1}{t}\sum_{\tau=1}^t \dgap(x_\tau, y_\tau) \rightarrow 0$, and best-iterate convergence: $\min_{\tau \leq t} \dgap(x_\tau, y_\tau) \rightarrow 0$.
As discussed, the focus of our work is to show how the additional opponent-action information helps speed up last-iterate convergence beyond the $t^{-\frac{1}{3}}$ (in-expectation) or $t^{-\frac{1}{4}}$ (high-probability) rate that are optimal in the absence of opponent-action feedback, a question that has not been explored in the literature to our knowledge.

We also recall the real-life examples in \pref{sec:intro} that illustrate the ubiquitousness of opponent-action information in practice.
In particular, we point out again that dueling bandits are a special case of our setting where the game matrix $A$ (known as the preference matrix in this case) is skew-symmetric ($A = -A^\top$), and the algorithm needs to select a pair of actions $(i_t, j_t)$ to solicit preference feedback from the user.
In this case, the information $(i_t, j_t)$ is obviously known to the algorithm itself.

Finally, since we will use multi-armed bandits as an illustrative special case in many subsequent discussions, we remind the reader that when the game matrix $A$ is $d$ by $1$, our setting recovers the standard stochastic $d$-armed bandit problem.
Since the game matrix is just a vector in this case, we will denote it by $\ell \in [-1,1]^d$ instead.
It is clear that the duality gap of a strategy $x \in \Delta_d$ in this case reduces to the (instantaneous) regret of this strategy compared to the best action, which we denote as $\Reg(x)=\inner{x,\ell}- \min_{i\in [d]} \ell_{i}$.



%% file: sections/challenges.tex
\section{The Challenges: Why Existing Approaches Fail}\label{sec:obstacle}

Before introducing our solution, this section provides an in-depth discussion on the difficulty of the problem and  why existing approaches in the bandit literature do not generalize to games, thereby highlighting both the need for a fundamentally different approach and the novelty of our work.

First and foremost, we recall the difficulty of achieving last-iterate convergence compared to other goals such as achieving low regret, even in the special case of multi-armed bandits.
Indeed, while it is well-known that a key challenge for any learning problems with partial-information feedback is the exploitation-exploration trade-off, 
ensuring last-iterate convergence creates even stronger tension between them.
This is because last-iterate convergence requires that in \textit{every} round, a reasonably good strategy is played, while achieving low regret only requires that in \textit{most} rounds, a good strategy is played.
For example, while optimism-based approaches, such as the classic Upper Confidence Bound (UCB) algorithm~\citep{auer2002adaptive}, ensures optimal regret in many bandit problems, 
the fact that it occasionally selects bad actions for exploration renders them incompatible with last-iterate convergence, as formally proven in~\citet{liu2024uniform}.

Naively, to play a good strategy in every round while ensuring sufficient exploration, one can mix the empirically best strategy with a uniform exploration strategy.
However, it is well-known that such a naive exploitation-exploration trade-off fails to achieve $t^{-\frac{1}{2}}$ convergence even for bandits.
Indeed, in \pref{app:NEUniform}, we prove that for our problem, this naive strategy only achieves $\dgap(x_t, y_t) = \order(d^{\frac{1}{2}}t^{-\frac{1}{4}})$ with high probability. 
We will use this strategy as one of the baselines in our experiments in \pref{sec:exp}.

So how do we achieve $t^{-\frac{1}{2}}$ convergence even in bandits?
To our knowledge, there are essentially two types of approaches in the literature: hard-elimination and soft-elimination.
We review both approaches and discuss why they do not generalize to games.

\paragraph{Hard-elimination}
\citet{liu2024uniform} study different variants of the hard-elimination approach from the literature, but they all share the same idea:
eliminate actions that, based on current observations, cannot be the best action (with high probability), and (almost) uniformly explore the remaining actions.
Such uniform exploration among plausibly good actions makes sure that in each round $t$, every action $i$ that is not eliminated yet has regret $\Reg(e_i)$ less than $\otil(\sqrt{d/t})$, meaning that it is always ``safe'' to play any of them.

When trying to generalize hard-elimination to games, the first challenge one immediately meets is that we are no longer searching for the best strategy over a finite space; indeed, while in bandits the best strategy can always be in the form of $e_i$ for some $i$, it is entirely possible that a game has no PSNE of the from $(e_i, e_j)$ for some $i$ and $j$.
Therefore, the search space in game becomes $\Delta_d \times \Delta_d$ instead of $[d] \times [d]$.

However, we emphasize that \textit{even when a PSNE exists, hard-elimination over the space $[d]\times [d]$ is still a bad idea.}
This is due to the following key obstacle:
estimating $\dgap(e_i, e_j) = \max_{h, k \in [d]} A_{ih} - A_{kj}$ for some pure strategy $(e_i, e_j)$ cannot be done by simply playing $(e_i, e_j)$ itself since this only yields samples of $A_{ij}$; instead, by definition, it requires playing $(e_i, e_h)$ and $(e_k, e_j)$ where $e_h$ and $e_k$ are the best response to $e_i$ and $e_j$ respectively that realize the maximum in the duality gap definition. 
Importantly, \textit{these strategies $(e_i, e_h)$ and $(e_k, e_j)$ themselves can be bad strategies with a large duality gap!}
To see this, simply consider a $3$ by $3$ example $A = \begin{bmatrix}0 &-1 &0 \\ 1 &0 &-\epsilon \\ 0 &\epsilon &0 \end{bmatrix}$
for some small and unknown $\epsilon > 0$.
The strategies $(e_2, e_3)$ and $(e_3, e_2)$ have a large duality gap of $1+\epsilon$, but if we eliminate them early, then there is no way to get samples from $A_{23}$ or $A_{32}$ to determine the value of $\epsilon$ and thus no way to estimate the duality gap of $(e_3, e_3)$ (which is $2\epsilon)$.
This is in sharp contrast to the bandit case where playing an action naturally yields samples that helps estimate the quality of this action itself.

\paragraph{Soft-elimination}
We use the term soft-elimination to refer to approaches that never completely eliminate any actions, but instead adjust the probability of selecting an action based on its estimated quality.
While there is no explicit study of this approach for achieving $t^{-\frac{1}{2}}$ last-iterate convergence in multi-armed bandits,
the idea is in fact implicitly described in several works on contextual bandits, starting from~\citet{agarwal2014taming}.
To facilitate discussions, we describe an algorithm based on this idea in \pref{alg:falcon}, which is most closely related to the FALCON algorithm of~\citet{simchi2022bypassing} for contextual bandits and can be seen as a variant of it specified to the non-contextual setting.

Specifically, \pref{alg:falcon} proceeds in epochs with doubling length.
Within each epoch $s$, a sampling strategy $x_s \in \Delta_d$ for this epoch is computed via $\argmin_{x\in\Delta_d} \innerp{x}{\wh{\ell}_{s}}+\gamma_s\sum_{i=1}^d\log\frac{1}{x_i}$,
where $\wh{\ell}_{s}$ is a loss estimator computed by simply averaging data from the last epoch, $\gamma_s>0$ is a learning rate,
and the part $\sum_{i=1}^d\log\frac{1}{x_i}$ is often known as the log-barrier regularizer, included to encourage exploration (since it explodes when $x$ concentrates on a subset of actions).
In other words, we solve an estimated log-barrier-regularized bandit problem in each epoch.
With this strategy $x_s$, we simply sample $i_t$ from $x_s$ for all rounds $t = 2^{s-1},\ldots,2^{s}-1$ within this epoch,
and finally use the observed data to compute the next loss estimate $\wh{\ell}_{s+1}$.

\begin{algorithm}[t]
\caption{FALCON for Stochastic Multi-Armed Bandits}
\label{alg:falcon}
\KwIn{Learning rates $\{\gamma_s\}_{s\geq 1}$.} 

Initialize loss estimator $\wh{\ell}_{1}$ to be the all-zero vector.

\For{epoch $s = 1, 2, \dots $}{
    Compute sampling strategy: $x_{s} =
    \arg\min_{x\in\Delta_d} \innerp{x}{\wh{\ell}_{s}}+\gamma_s\sum_{i=1}^d\log\frac{1}{x_i}$.
    
    Initialize counters: $n_{s,i}=0$ for all $i\in[d]$.
    
    \For{round $t = 2^{s-1},\ldots,2^{s}-1$}{
        Sample $i_t\sim x_s$ and observe realized loss $r_t$ such that $\E[r_t]=\ell_{i_t}$.
        
        Increment counter $n_{s,i_t} \gets n_{s,i_t} + 1$.
    }
    
    Compute loss estimate $\wh{\ell}_{s+1,i}=\frac{\mathbbm{1}\cbr{n_{s,i}\neq0}}{n_{s,i}}\sum_{\tau=2^{s-1}}^{2^{s}-1} \mathbbm{1}\cbr{i_{\tau} = i}\cdot r_{\tau}$ for all $i\in[d]$.
}
\end{algorithm}

Other than specifying the FALCON algorithm to the non-contextual setting and instantiating its regression oracle with the simple empirical mean estimators,
the only slight difference between \pref{alg:falcon} and FALCON is how $x_s$ is defined, but it is known that our definition shares the same properties as theirs; see e.g.~\citet{foster2020beyond, foster2020adapting, zhang2024contextual}.
Our definition is also closely related to the well-known Follow-the-Regularized-Leader (FTRL) framework, and using log-barrier regularizer in this framework for bandits has been extensively studied since the early work such as~\citet{foster2016learning, wei2018more}.
However, an important distinction is that these FTRL algorithms all use importance-weighted estimators instead of the empirical mean estimators that \pref{alg:falcon} uses.

The key idea in the analysis for such soft-elimination methods is to show that, with an appropriate choice of $\gamma_s$ (that is of order $1/\sqrt{d 2^s}$),  for each action $i$, its regret $\Reg(e_i)$ is close to the estimated version $\wh{\Reg}_s(e_i)$, where
$\wh{\Reg}_s(x) = \innerp{x}{\wh{\ell}_s}-\min_{i\in[d]}\wh{\ell}_{s,i}$, in a multiplicative sense: $\Reg(e_i) \leq 2 \wh{\Reg}_s(e_i) + \otil(d\gamma_s)$ and similarly $\wh{\Reg}_s(e_i) \leq 2\Reg(e_i) + \otil(d\gamma_s)$.
This is enabled by the following key low-regret-low-variance property of the algorithm that underscores the essence of a line of previous work on contextual bandits; see e.g., ~\citet[Algorithm~1]{agarwal2014taming}, \citet[Observation~2]{simchi2022bypassing}, and~\citet[Lemma~6]{zhang2024contextual}.

\begin{lemma}[Low-Regret-Low-Variance]
\label{lem:low-reg-low-var}
    For every epoch $s\ge 1$, \pref{alg:falcon} satisfies the following:
    \begin{equation}\label{eq:low-reg}
        \wh{\Reg}_s(x_s) \le d\gamma_s ,
    \end{equation}
    \begin{equation}\label{eq:low-var}
        \frac{1}{x_{s,i}}\le d + \frac{\wh{\Reg}_s(e_i)}{\gamma_s},~\forall~i\in[d].
    \end{equation}
\end{lemma}

Note that $x_{s,i}$ is the probability of getting a sample for action $i$ in each round of epoch $s$ and thus $1/x_{s,i}$ controls the variance of the loss estimate for this action.
\pref{eq:low-var} therefore states that the variance for action $i$ is controlled by its estimated regret: the better the action appears to be, the lower the estimation variance is,
which is exactly the key to show the multiplicative closeness between $\Reg(e_i)$ and $\wh{\Reg}_s(e_i)$.
Such a low-variance property holds thanks to the exploration induced by the special log-barrier regularization.
On the other hand, due to such closeness, the low-estimated-regret property of \pref{eq:low-reg} thus immediately implies that the true per-round regret of the algorithm is also small, hence achieving $\otil(d\gamma_s) = \otil(\sqrt{d/t})$ last-iterate convergence~\citep[Lemma~9]{simchi2022bypassing}.

To understand the difficulty of generalizing such soft-elimination methods to games, let us consider again the simpler case with a PSNE.
The most obvious generalization would be to try to ensure multiplicative closeness between $\dgap(e_i, e_j)$ and its estimated version for any action pair $(e_i, e_j)$.
This is impossible due to the exactly same obstacle mentioned before: estimating $\dgap(e_i, e_j)$ requires playing other potentially bad strategies.
Indeed, \citet{saha2022efficient} already use the same aforementioned $3$ by $3$ example to show that if such multiplicative closeness holds, then the algorithm cannot have $\sqrt{t}$ regret, let alone $t^{-\frac{1}{2}}$ last-iterate convergence.
In \pref{app:low-reg-low-var}, we provide more discussion on why a low-duality-gap-low-variance property similar to \pref{lem:low-reg-low-var}, which our final algorithm \textit{does} ensure, fails to show $t^{-\frac{1}{2}}$ last-iterate convergence.

We point out that \citet{saha2025efficient} make an attempt to get around this obstacle in the contextual dueling bandit setting by imposing a strong "Idempotent Best-Response" assumption.
When specified to our setting, this would mean that their results only hold for a special subset of games with a PSNE.
Also, their algorithm requires performing least squares regression over this class of games, which is unlikely to be computationally efficient. 

%% file: sections/algorithm.tex
\section{Our Algorithm and Analysis}\label{sec:alg}

Given all the obstacles discussed above, it might seem that all hope is lost.
Nevertheless, we propose a solution that algorithmically follows the soft-elimination approach but analytically requires a different and novel analysis. 
Specifically, our algorithm, called Phased Minimax Optimization with Log-Barrier Regularization (PMO-LB), is shown in \pref{alg:game_falcon}.
It follows \pref{alg:falcon} with the same epoch schedule and empirical mean estimator $\wh{A}_{s,ij}$ for each $(i,j)$ entry of the unknown matrix $A$,
and generalizes it in how the sampling strategy is computed: $(x_s, y_s)$ is the Nash equilibrium of the estimated log-barrier-regularized game specified by $\Phi_s$ in \pref{eq:Phi}. Equivalently, $x_s = \argmin_{x\in \Delta_d}\max_{y\in \Delta_d} \Phi_s(x, y)$ and $y_s = \argmax_{y\in \Delta_d}\min_{x\in \Delta_d} \Phi_s(x, y)$.
Such a (unique) Nash equilibrium must exist since $\Phi_s$ is strongly convex in $x$ and strongly concave in $y$, and can be computed efficiently using standard methods from the saddle point optimization literature.
When the game matrix is $d$ by $1$, it is clear that our algorithm reduces to \pref{alg:falcon}.
Before analyzing our algorithm, we make the following important remarks.

\paragraph{Uncoupled/coupled dynamics}
While we present both players' algorithm together in \pref{alg:game_falcon}, 
it is clear that these two players can compute their own strategy \textit{independently without any extra coordination/coupling/shared randomness}.
Indeed, thanks to the opponent-action information, each player can construct the estimator $\wh{A}_s$ and compute their strategy on their own.
Whether this is considered as an uncoupled or coupled learning dynamic depends on how one define them; for example, \citet{fiegel2025harder} define uncoupled learning dynamics as those that do not have the opponent-action information.

\paragraph{Connection to existing algorithms}
Other than being a generalization of \pref{alg:falcon} (which itself is a variant of FALCON~\citep{simchi2022bypassing}),
our algorithm is also connected to several others in the literature.
For example, \citet{saha2025efficient} also generalize the FALCON algorithm to contextual dueling bandits.
Their algorithm shares similarities with ours but importantly computes a \textit{joint} distribution over action pairs and thus requires coupling when sampling actions.
Moreover, as mentioned in \pref{sec:obstacle}, their results hold only under a strong assumption, which is inevitable since their analysis still follows the same low-regret-low-variance approach discussed in \pref{sec:obstacle}, proven to fail generally according to~\citet{saha2022efficient}.
On the other hand, our algorithm also shares similarity with~\citet[Algorithm 3]{cardoso2019competing}, with the following differences: first, they use importance-weighted estimators instead of empirical mean estimators; second, they use entropy regularizer instead of the log-barrier regularizer (the latter is critical for our analysis as it will become clear);
finally, their algorithm is designed for a quite different problem with time-varying game matrices, and the metric they study is a notion called NE regret instead of last-iterate convergence.
Finally, we note that the algorithm of~\citet{fiegel2026optimal} for the setting without opponent-action information also uses log-barrier regularization but in the FTRL framework and does not involve minimax optimization.
Their algorithm is closer to that of~\citet{cai2023uncoupled} with the difference being the form of the regularizer.

\begin{algorithm}[t]
\caption{Phased Minimax Optimization with Log-Barrier Regularization (PMO-LB)}
\label{alg:game_falcon}
\KwIn{Learning rate $\{\gamma_s\}_{s\geq 1}$.} 

Initialize game estimator $\wh{A}_1$ to be the all-zero $d$ by $d$ matrix.


\For{$s = 1, 2, \dots $}{
    Define an estimated log-barrier-regularized game with the following payoff function:
    \begin{equation}\label{eq:Phi}
    \Phi_s(x,y) \triangleq x^\top\wh{A}_sy + \gamma_s\sum_{i=1}^d\log\frac{1}{x_i} - \gamma_s\sum_{j=1}^d\log\frac{1}{y_j}.
    \end{equation}

    For row player: compute $x_s = \argmin_{x\in \Delta_d}\max_{y\in \Delta_d} \Phi_s(x, y)$.

    For column player: compute $y_s = \argmax_{y\in \Delta_d}\min_{x\in \Delta_d} \Phi_s(x, y)$.

    Initialize counters: $n_{s,ij}=0$ for all $i,j\in[d]$.
    
    \For{$t=2^{s-1},\dots,2^{s}-1$}{
        Row player samples $i_t\sim x_s$ and column player samples $j_t\sim y_s$.
        
        Both players observe $(i_t, j_t)$ and $r_t$ with $\E[r_t]=A_{i_tj_t}$.
        
        Increment counter $n_{s,i_tj_t}\leftarrow n_{s,i_tj_t}+1$.
    }

    Compute game estimator $\wh{A}_{s+1,ij}=\frac{\mathbbm{1}\cbr{n_{s,ij}\neq 0}}{n_{s,ij}}\sum_{\tau=2^{s-1}}^{2^{s}-1}r_\tau\cdot \mathbbm{1}\{i_\tau=i,j_\tau=j\}$ for all $i,j\in [d]$.
}
\end{algorithm}

\subsection{Main Results and Analysis}\label{sec:result-analysis}
Our main result is the following theorem.
\begin{restatable}{theorem}{dgapGame}
\label{thm:dgapGame}
For any fixed $\delta \in (0,1)$,
\pref{alg:game_falcon} with $\gamma_s = 128d\cdot 2^{-s/2}\sqrt{\log(8d^2s^2/\delta)}$ guarantees that with probability at least $1-\delta$, 
$\dgap(x_t,y_t) = \order\rbr{d^2\sqrt{\frac{\log(d\log t/\delta)}{t}}}$ for all $t\geq 1$, where $(x_t,y_t)$ denotes the strategy pair played in round $t$.
\end{restatable}

As mentioned, the best existing results prior to our work are those that do not utilize the opponent-action information: $t^{-\frac{1}{3}}$ (in-expectation) from~\citet{fiegel2025harder} and $t^{-\frac{1}{4}}$ (high-probability) from~\citet{fiegel2026optimal}.
Our work is the first to show that with opponent-action feedback, $t^{-\frac{1}{2}}$ last-iterate convergence is achievable, eliminating the gap between last-iterate and average-iterate convergence, at least in terms of the dependence on $t$.
As for the dependence on $d$, the best lower bound is the standard $\Omega(\sqrt{d/t})$ for average-iterate convergence, coming from the bandit literature~\citep{auer2002nonstochastic}.
We are unaware of tighter lower bounds for last-iterate convergence with opponent-action information, but we suspect that our $d^2$ dependence is suboptimal.

Despite our algorithm being a direct generalization of \pref{alg:falcon}, as mentioned, our analysis is (and has to be) different to overcome the obstacles discussed in \pref{sec:obstacle}.
To illustrate our core ideas, we discuss the single-player case again as a warm-up, which can be seen as an alternative analysis of \pref{alg:falcon}. (All omitted proof can be found in the appendix).

\paragraph{Warm-up for the single-player case}
Our starting point is the following property that is strictly stronger than the low-regret-low-variance property of \pref{lem:low-reg-low-var}.

\begin{lemma}
\label{lem:estimated-reg-bound}
    For every epoch $s\ge 1$, \pref{alg:falcon} satisfies
    \begin{equation}\label{eq:est-reg}
        \innerp{x_s}{\wh{\ell}_s}-\wh{\ell}_{s,i} \le d\gamma_s  - \frac{\gamma_s}{x_{s,i}},~\forall~i\in[d].
    \end{equation}
\end{lemma}
This is obtained by simply applying the first-order optimality condition to the optimization objective that defines $x_s$.
It is stronger than \pref{lem:low-reg-low-var} since it trivially implies \pref{eq:low-reg} by dropping the negative term $- \frac{\gamma_s}{x_{s,i}}$,
and it also implies \pref{eq:low-var} by rearranging terms and realizing $ \wh{\ell}_{s,i} - \innerp{x_s}{\wh{\ell}_s} \leq \wh{\ell}_{s,i} - \min_{j\in[d]} \wh{\ell}_{s,j} = \wh{\Reg}_s(e_i)$.
This new lemma is useful since it asserts that 
the algorithm's estimated regret compared against a particular action is much smaller when this action is selected with a low probability, which in turn compensates the large estimation variance for this action.
Indeed, consider bounding the true regret of $x_s$ as follows (with $i^\star\in\argmin_{i\in[d]}\ell_i$ being an optimal action):
\begin{align*}
\Reg(x_s) 
&= \innerp{x_s}{\wh{\ell}_s}-\wh{\ell}_{s,i^\star} + \innerp{x_s - e_{i^\star}}{\ell-\wh{\ell}_s} 
\leq d\gamma_s - \frac{\gamma_s}{x_{s,i^\star}}
+ \frac{\beta_s}{\sqrt{x_{s-1,i^\star}}} + \sum_{i=1}^d \frac{\beta_s x_{s,i}}{\sqrt{x_{s-1,i}}},
\end{align*}
where the inequality uses \pref{lem:estimated-reg-bound} and standard concentration to bound $\abs{\wh{\ell}_{s,i}-\ell_i}$ for each $i$ by $\frac{\beta_s}{\sqrt{x_{s-1,i}}}$ with $\beta_s = \otil(1/\sqrt{2^s})$, since $\wh{\ell}_{s,i}$ is constructed using samples collected from strategy $x_{s-1}$. 
At this point, it is clear that if $x_{s}$ and $x_{s-1}$ are close enough in a multiplicative sense, that is, if $x_{s,i} \leq \alpha x_{s-1,i}$ holds for all $i$ and some $\alpha = \order(\text{poly}(d))$,
then we can continue with
\[
\Reg(x_s) \leq d\gamma_s - \frac{\gamma_s}{\alpha x_{s-1,i^\star}}
+ \frac{\beta_s}{\sqrt{x_{s-1,i^\star}}} + \sum_{i=1}^d \frac{\beta_s \sqrt{\alpha}x_{s,i}}{\sqrt{x_{s,i}}}
\leq d\gamma_s + \frac{\alpha\beta_s^2}{4\gamma_s} + \beta_s \sqrt{\alpha d},
\]
where we use AM-GM inequality: $- \frac{\gamma_s}{\alpha x_{s-1,i^\star}}
+ \frac{\beta_s}{\sqrt{x_{s-1,i^\star}}} \leq \frac{\alpha\beta_s^2}{4\gamma_s}$, illustrating again the importance of the negative term from \pref{eq:est-reg} in compensating the potentially large estimation error $\frac{\beta_s}{\sqrt{x_{s-1,i^\star}}}$.
This concludes the $t^{-\frac{1}{2}}$ last-iterate convergence after setting $\gamma_s$ appropriately. 
It thus remains to prove the multiplicative stability, which, thanks to the log-barrier regularization, indeed holds with $\alpha=\order(d)$; see \pref{lem:mult-stability} in the appendix.
Although it is known that log-barrier regularization ensures multiplicative stability in the FTRL framework, as first shown by~\citet{wei2018more} and later utilized in many other works, our proof is different and needs to resolve several new challenges, including handling the time-varying learning rate $\gamma_s$ and understanding how $\wh{\ell}_s$ deviates from $\wh{\ell}_{s-1}$, the latter of which itself requires using multiplicative stability from previous epochs.
 
Formally, we show a $\otil(d/\sqrt{t})$ last-iterate convergence rate for \pref{alg:falcon} using this new analysis (see \pref{prop:falcon-convergence}), slightly worse than the $\otil(\sqrt{d/t})$ rate obtained from the analysis discussed in \pref{sec:obstacle}.
However, the benefit is that this new analysis immediately generalizes to games (unlike the one in \pref{sec:obstacle} that provably cannot).

\paragraph{Generalization to games} 

The generalization starts with the following analogue of \pref{lem:estimated-reg-bound}. 

\begin{restatable}
{lemma}{dgapEstGame}
\label{lem:dgapEstGame}
  For every epoch $s\geq 1$, \pref{alg:game_falcon} satisfies: 
  \begin{align}\label{eqn:dgapEstGame}
      x_s^\top \wh{A}_{s}e_j - e_i^\top \wh{A}_{s}y_s \leq 2d\gamma_s  - \frac{\gamma_s}{x_{s,i}} - \frac{\gamma_s}{y_{s,j}}, \;\forall i, j\in [d].
  \end{align}
\end{restatable}

This lemma shows that the lower the sampling probability of an action, the smaller the increase in payoff one can obtain by deviating unilaterally to this action.
We note that Eq.~(2) of~\citet{saha2025efficient} is  close to our \pref{eqn:dgapEstGame}, but theirs involves a joint distribution over all action pairs and is thus fundamentally different from ours.
Next, we generalize the multiplicative stability as follows.

\begin{restatable}[Multiplicative stability]{lemma}{stabGame_concise}
\label{lem:stabGame_concise}
    With probability at least $1-\delta$, \pref{alg:game_falcon} with $\gamma_s$ set to the value stated in \pref{thm:dgapGame} ensures
    for all $s\geq 2$ and $i,j\in[d]$,
    $x_{s,i}\leq 16d x_{s-1,i}$, 
    $y_{s,j}\leq 16d y_{s-1,j}$.
\end{restatable}

The proof of this lemma is the most technical part of our analysis.
It not only needs to address all the issues mentioned earlier in the single-player case, but also needs to take into account the fact that $x_s$ is the solution of an optimization defined in terms of $y_s$ (and vice versa) and both of them are changing over time.
With these two key lemmas, the proof of our main result \pref{thm:dgapGame} can be proven following  similar ideas outlined earlier for the single-player case; see \pref{app:falcon_game} for details.

%% file: sections/experiments.tex
\section{Experiments}\label{sec:exp}
\paragraph{Experimental setup}
We evaluate the performance of PMO-LB (\pref{alg:game_falcon}) against a set of baselines on two security games from the benchmark of~\citet{krever2025guard} for $10^7$ rounds.
The two games are of size $61 \times 21$ and $39 \times 35$ respectively.
%
The five baselines we compared against are: 
\citet[Algorithm~1]{cai2023uncoupled}, 
\citet[Algorithm~2]{cai2025average},
\citet[Algorithm~1]{fiegel2025harder},
\citet[Algorithm~1]{fiegel2026optimal},
and the naive baseline (\pref{alg:ne-estimate}) analyzed in \pref{app:NEUniform}.
Note that only our \pref{alg:game_falcon} and \pref{alg:ne-estimate} utilize access to the opponent’s actions, whereas the remaining algorithms do not.

\paragraph{Implementation details}
We generate the two game matrices using the GUARD Green Security Game environment~\citep{krever2025guard}.  In both instances, the geographic region is discretized into a $7 \times 7$ grid, and target values are generated by the centroid-based method with $6$ clusters. The first matrix uses $5$ stationary attackers, one mobile defender, home base $(2.15,16.02)$, and horizon $5$, whereas the second matrix uses $4$ stationary attackers, one mobile defender, home base $(2.27,16.0)$, and horizon $5$. 

Hyperparameters of all algorithms are chosen according to theoretical requirements  where available, and otherwise tuned to ensure stable performance. 
Results are averaged over $10$ independent runs for each game.
All experiments were conducted on a Linux server running Ubuntu 24.04.3 LTS. The server was equipped with two Intel Xeon Gold 6530 processors, providing 64 physical CPU cores and 128 logical threads, and 1.0 TiB of RAM. The code was implemented in Python 3.13.12.

\begin{figure}[htbp]
    \centering
    \begin{subfigure}[b]{0.47\textwidth}
        \centering
        \includegraphics[width=\textwidth]{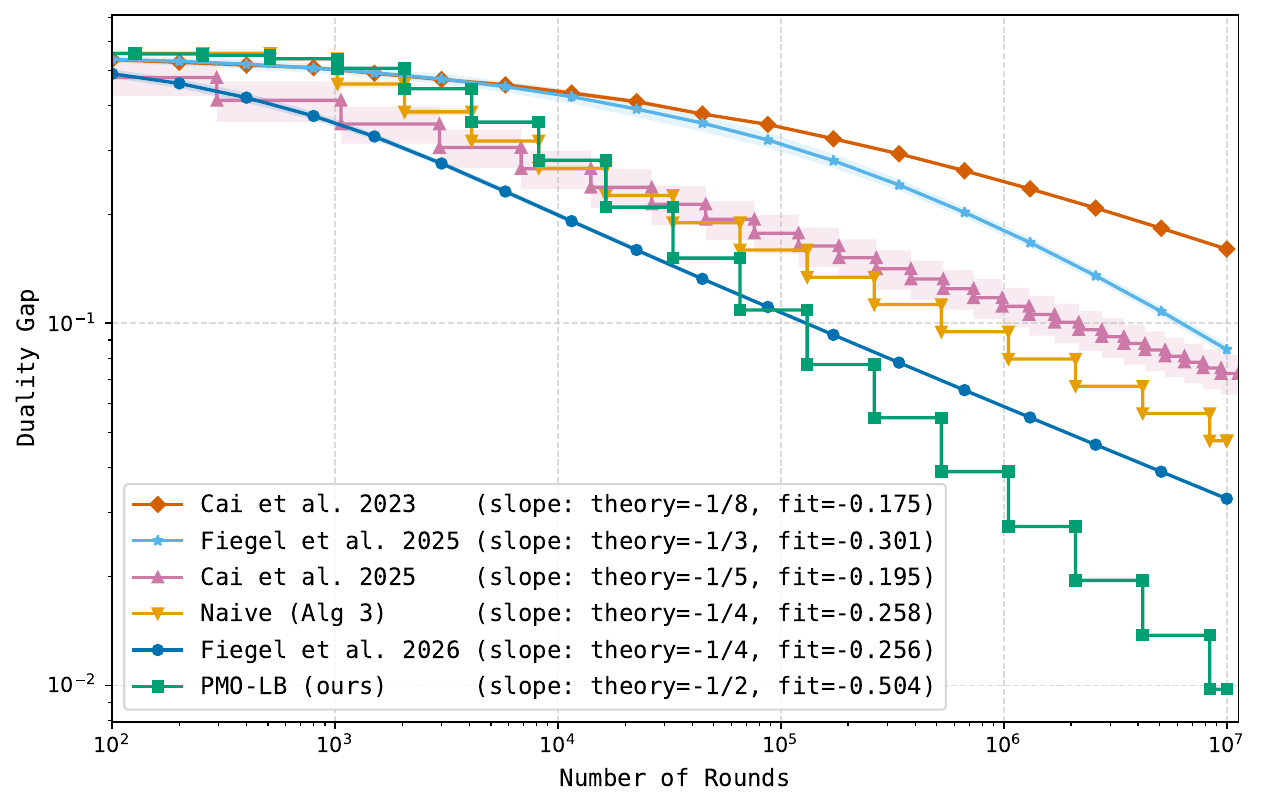}
        \caption{Plot for a security game of size $61\times 21$}
        \label{fig:plot-61x21}
    \end{subfigure}
    \hfill
    \begin{subfigure}[b]{0.47\textwidth}
        \centering
        \includegraphics[width=\textwidth]{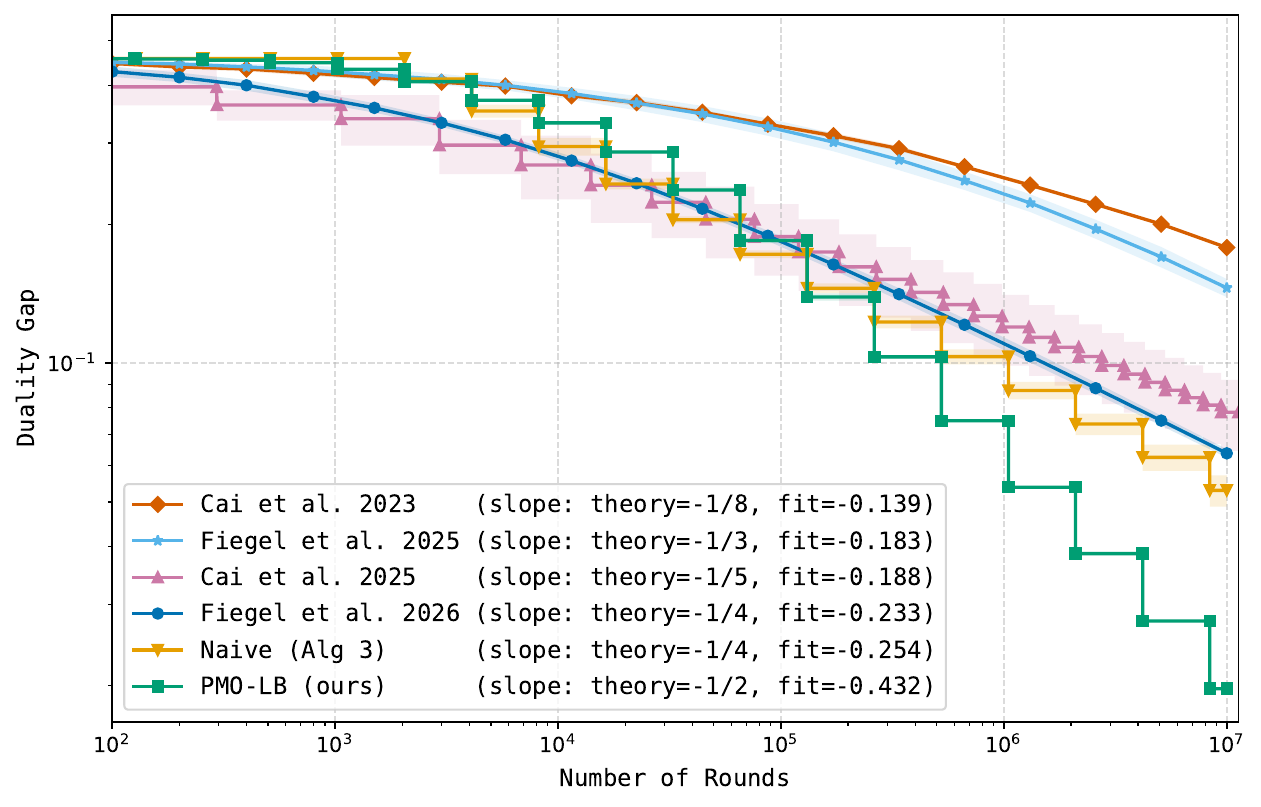}
        \caption{Plot for a security game of size $39\times 35$}
        \label{fig:plot-39x35}
    \end{subfigure}
    \caption{Plots of duality gap vs. number of rounds for PMO-LB and five baseline algorithms}
    \label{fig:plots}
\end{figure}

 \paragraph{Results and discussion}
 \pref{fig:plots} compares the duality gap as a function of the number of rounds across all algorithms, in log-log scale. 
 We also include in the legend the slope (that is, the exponent of the convergence rate) of each plot as predicted by the theory and as fitted based on the actual results starting from $t=10^4$.
 For most algorithms, the predicted slope indeed aligns with the empirical one.
 More importantly, the figure shows that 
 our PMO-LB algorithm consistently outperforms all other baselines.


%% file: appendix/omitted_proofs.tex
\newpage
\section{Omitted Details for Multi-Armed Bandits}\label{app:alg}

In this section, we analyze the single-agent setting from \pref{sec:alg} and formally state the convergence guarantee for FALCON under our approach.

\begin{proposition}[Last-iterate convergence of FALCON]\label{prop:falcon-convergence}
    For any fixed $\delta \in (0,1)$,
    \pref{alg:falcon} with $\gamma_s = 40\cdot 2^{-s/2}\sqrt{\log(8ds^2/\delta)}$ guarantees that with probability at least $1-\delta$,  
    $\Reg(x_t)=\order\rbr{d\sqrt{\frac{\log(d\log t/\delta)}{t}}}$ for all $t\geq 1$, where $x_t$ denotes the strategy played in round $t$.
\end{proposition}

We begin by showing a high-probability bound on the estimation error of each arm’s loss under the epoch-wise sampling scheme of \pref{alg:falcon}.

\begin{definition}\label{def:closeMAB}
     Define $\calE_1$ to be the event: $\calE_1\triangleq \cbr{ \forall s\ge 2, \forall i\in [d], \abs{\wh{\ell}_{s,i}-\ell_i} \le \frac{\beta_s}{\sqrt{x_{s-1,i}}}}$ where $\beta_s = \sqrt{\frac{16\log(8ds^2/\delta)}{2^{s-2}}}$.
\end{definition}

\begin{lemma}[Estimation error] \label{lem:estimation-error}
    \pref{alg:falcon} guarantees that $\calE_1$ holds with probability at least $1-\delta$.
\end{lemma}

First, we state a standard concentration inequality that will be used in the proof of the estimation error bound.

\begin{lemma}[Multiplicative Chernoff bound]\label{lem:mult-chernoff}
    Let $X = \sum_{i=1}^n X_i$, where $X_i$ is a Bernoulli random variable with mean $p_i$ and all $X_i$ are independent. Let $\mu = \E[X] = \sum_{i=1}^n p_i$. Then
    \[
        \Pr\cbr{X \le (1 - \delta)\mu} \le e^{-\mu\delta^2/2},\quad \forall~\delta\in(0,1).
    \]
\end{lemma}

\begin{proof}[Proof of {\pref{lem:estimation-error}}]
    Fix $s \ge 2$. Conditional on the history prior to epoch $s-1$, the sampling distribution $x_{s-1}$ is fixed. The length of epoch $s-1$ is $2^{s-2}$, and hence the expected number of times arm $i \in [d]$ is selected in epoch $s-1$ is $2^{s-2} x_{s-1,i}$. We first show that for any $\xi_s \in (0,1)$, with probability at least $1 - \xi_s$, the inequality $|\wh{\ell}_{s,i}-\ell_i| \le \sqrt{\frac{16\log(4d/\xi_s)}{2^{s-2}x_{s-1,i}}}$ holds simultaneously for all $i \in [d]$.
    
    For any arm $i\in[d]$, if $2^{s-2}x_{s-1,i} \le 16 \log(4d/\xi_s)$, then  $\sqrt{\frac{16\log(4d/\xi_s)}{2^{s-2}x_{s-1,i}}}\ge 1$. Since both $\wh{\ell}_{s,i}$ and $\ell_i$ lie in $[0,1]$, the bound holds trivially.

    Now suppose $2^{s-2}x_{s-1,i} > 16 \log(4d/\xi_s)$. Applying the multiplicative Chernoff bound in \pref{lem:mult-chernoff}, we obtain 
    \begin{align*}
        \Pr\cbr{n_{s-1,i}<\frac{2^{s-2}x_{s-1,i}}{2}}
        \le e^{-\frac{2^{s-2}x_{s-1,i}}{8}} \le e^{-2\log(4d/\xi_s)}.
    \end{align*}

    Conditioned on $n_{s-1,i}=n\ge 2^{s-2}x_{s-1,i}/2$, the estimate $\wh{\ell}_{s,i}$ is an average of $n$ independent random variables in $[0,1]$ with mean $\ell_i$. Thus, by Hoeffding's bound,
    \begin{align*}
        \Pr\cbr{\abs{\wh{\ell}_{s,i}-\ell_i} > \sqrt{\frac{16\log(4d/\xi_s)}{2^{s-2}x_{s-1,i}}} ~\middle|~ n_{s-1,i}=n} &\le 2\exp\rbr{-2n\frac{16\log(4d/\xi_s)}{2^{s-2}x_{s-1,i}}}\\&\le 2e^{-16\log(4d/\xi_s)}.
    \end{align*}

    Combining the above bounds, for each fixed arm $i$, 
    \begin{align}\label{eqn:concentration}
        \Pr\cbr{\abs{\wh{\ell}_{s,i}-\ell_i} > \sqrt{\frac{16\log(4d/\xi_s)}{2^{s-2}x_{s-1,i}}}} \le e^{-2\log(4d/\xi_s)}+2e^{-16\log(4d/\xi_s)} \le \frac{\xi_s}{d}.
    \end{align}

    A union bound over all $d$ arms gives us
    \begin{equation}\label{eq:per-epoch-concentration}
        \Pr\cbr{\forall i\in [d], \abs{\wh{\ell}_{s,i}-\ell_i} \le \sqrt{\frac{16\log(4d/\xi_s)}{2^{s-2}x_{s-1,i}}}} \ge 1-\xi_s.
    \end{equation}

    Setting $\xi_s = \frac{\delta}{2s^2}$ in \pref{eq:per-epoch-concentration} gives that $\abs{\wh{\ell}_{s,i}-\ell_i} \le \frac{\beta_s}{\sqrt{x_{s-1,i}}}$ holds for all $i \in [d]$ with probability at least $1 - \xi_s$.

    Finally, applying a union bound over $s \ge 2$ yields a total failure probability of at most $\sum_{s=2}^{\infty}\frac{\delta}{2s^2}\le\delta$. Therefore, the event $\calE_1$ defined in \pref{def:closeMAB}  holds with probability at least $1-\delta$.
\end{proof}

Next, we prove \pref{lem:estimated-reg-bound}, which is a stronger version of \pref{lem:low-reg-low-var}.

\begin{proof}[Proof of {\pref{lem:estimated-reg-bound}}]
    The sampling strategy $x_s$ is obtained by solving the convex optimization problem $F(x)\triangleq\min_{x\in\Delta_d} \innerp{x}{\wh{\ell}_{s}}+\gamma_s\sum_{i=1}^d\log\frac{1}{x_i}$.
    With a slight abuse of notation, for $v \in \R^d$, we denote by $\frac{1}{v} \in \R^d$ the vector whose $i$-th coordinate is $\frac{1}{v_i}$. By first-order optimality, for any $x \in \Delta_d$,
    \begin{align*}
        \inner{x-x_s,\nabla F(x_s)} \ge 0\implies \inner{x-x_s,\wh{\ell}_s-\frac{\gamma_s}{x_s}}\ge 0.
    \end{align*}
    Choosing $x = e_i$ gives us the desired inequality.
\end{proof}

We next state a multiplicative stability bound for successive strategies.
\begin{lemma}[Multiplicative stability]
\label{lem:mult-stability}
    With probability at least $1-\delta$, \pref{alg:falcon} with $\gamma_s$ set to the value stated in \pref{prop:falcon-convergence} ensures
    for all $s\geq 2$ and $i\in[d]$,
    $x_{s,i}\leq 16d x_{s-1,i}$.
\end{lemma}

First, we prove the more general stability result below, and \pref{lem:mult-stability} follows from it by choosing $\gamma_s$ appropriately, as shown in the proof of \pref{prop:falcon-convergence}.

\begin{lemma}
\label{lem:mult-stability-induction}

    Suppose that $\calE_1$ holds and define $\rho_s=\gamma_s/\gamma_{s+1}$. Further suppose that the sequence
    $\{\gamma_s\}_{s\geq 1}$ in \pref{alg:falcon} satisfies (1) for all $s\geq 2$,
    \(
    2(\rho_s-1)d
    +
    \frac{
    \rho_s^2(1+\sqrt{\alpha})^2\max\{\beta_s,\beta_{s+1}\}^2
    }{\gamma_s^2}
    \leq
    \left(\sqrt{\alpha}-\frac{1}{\sqrt{\alpha}}\right)^2
    \)
    for certain $\alpha\geq d$. (2) $\rho_s\in[1,\sqrt{2}]$. Then \pref{alg:falcon}
    guarantees that, for all $s\geq 2$,
    $
        x_{s,i}\leq \alpha x_{s-1,i},\forall i\in[d].
    $
\end{lemma}

\begin{proof}[Proof of {\pref{lem:mult-stability-induction}}]
    We prove using induction on $s$. The base case $s=2$ follows since $x_1=\one/d$ and
    \[
        x_{2,i}\leq 1\leq d x_{1,i}\leq \alpha x_{1,i},
    \]
    as $\alpha \ge d$.

    Assume that for some $s\geq 2$, $x_{s,i}\leq \alpha x_{s-1,i}$, for all $i\in[d]$.
    We show that this also holds for $s+1$. Define
    \[
        \bar\beta_s=\max\{\beta_s,\beta_{s+1}\},
        \quad
        \epsilon_s=(1+\sqrt{\alpha})\bar\beta_s.
    \]
    Under event $\calE_1$, applying \pref{lem:estimation-error} and the induction hypothesis, for every $i\in[d]$,
    \begin{align}
        \abs{\wh \ell_{s,i}-\wh \ell_{s+1,i}}
        &\leq
        \abs{\wh \ell_{s,i}-\ell_{i}}
        +
        \abs{\wh \ell_{s+1,i}-\ell_{i}} \\&\leq \frac{\beta_s}{\sqrt{x_{s-1,i}}}+\frac{\beta_{s+1}}{\sqrt{x_{s,i}}} \\&\leq
        \frac{\sqrt{\alpha}\beta_s+\beta_{s+1}}{\sqrt{x_{s,i}}}
        \leq
        \frac{\epsilon_s}{\sqrt{x_{s,i}}}.
        \label{eqn:ell-diff-stab}
    \end{align}

    Define 
    \[
        Q_s\triangleq\sum_{i=1}^d\frac{(x_{s+1,i}-x_{s,i})^2}{x_{s+1,i}x_{s,i}}.
    \]

    By first-order optimality at epochs $s$ and $s+1$, for any $x \in \Delta_d$,
    \begin{align}
        \inner{\wh{\ell}_s-\frac{\gamma_s}{x_s},x-x_s}\geq 0,
        \label{eqn:first-order-x-r-mab}\\
        \inner{\wh{\ell}_{s+1}-\frac{\gamma_{s+1}}{x_{s+1}},x-x_{s+1}}\geq 0.
        \label{eqn:first-order-x-r-plus-one-mab}
    \end{align}
    Setting $x=x_{s+1}$ in \pref{eqn:first-order-x-r-mab} and $x=x_s$ in \pref{eqn:first-order-x-r-plus-one-mab}, and then summing them, gives
    \[
        Q_s
        \leq
        \inner{x_{s+1}-x_s,
        \frac{\wh \ell_s}{\gamma_s}
        -
        \frac{\wh \ell_{s+1}}{\gamma_{s+1}} }.
    \]
    We have $\rho_s=\gamma_s/\gamma_{s+1}$. 
    Thus, we get
    \begin{align}
        Q_s
        &\leq 
        \frac{\rho_s}{\gamma_s}
        \inner{x_{s+1}-x_s,\wh\ell_s-\wh \ell_{s+1}}
        +
        \frac{1-\rho_s}{\gamma_s}
        \inner{x_{s+1}-x_s,\wh \ell_s}.
        \label{eqn:x-stability-decomp-mab}
    \end{align}

    From \pref{eqn:first-order-x-r-mab} with $x=x_{s+1}$,
    \[
        \inner{x_{s+1}-x_s,\frac{\wh \ell_s}{\gamma_s}}
        \geq
        \inner{x_{s+1}-x_s,\frac{1}{x_s}}.
    \]
    Since $\rho_s\geq 1$, this implies
    \begin{align*}
        \frac{1-\rho_s}{\gamma_s}
        \inner{x_{s+1}-x_s,\wh \ell_s}
        &\leq
        (1-\rho_s)
        \inner{x_{s+1}-x_s,\frac{1}{x_s}} \\
        &=
        (\rho_s-1)
        \left(
        d-\inner{\frac{1}{x_s},x_{s+1}}
        \right)
        \leq
        (\rho_s-1)d.
    \end{align*}

    For the first term, using \pref{eqn:ell-diff-stab},
    \begin{align*}
        \abs{\innerp{x_{s+1}-x_s}{\wh \ell_s-\wh \ell_{s+1}}}
        &\leq \sum_{i=1}^d \abs{x_{s+1,i}-x_{s,i}}\abs{\wh \ell_{s,i}-\wh \ell_{s+1,i}} \\
        &\leq \epsilon_s \sum_{i=1}^d
        \frac{|x_{s+1,i}-x_{s,i}|}{\sqrt{x_{s,i}}} \\
        &= \epsilon_s
        \sum_{i=1}^d
        \frac{|x_{s+1,i}-x_{s,i}|}{\sqrt{x_{s+1,i}x_{s,i}}}
        \sqrt{x_{s+1,i}} \\
        &\leq \epsilon_s\sqrt{Q_s},
    \end{align*}
    where the last step follows from the Cauchy-Schwarz inequality.

    Combining the bounds,
    \[
        Q_s \leq (\rho_s-1)d + \frac{\rho_s\epsilon_s\sqrt{Q_s}}{\gamma_s}.
    \]

    Using $a\sqrt{Q_s}\leq (Q_s+a^2)/2$ with $a=\frac{\rho_s\epsilon_s}{\gamma_s}$,
    \[
        Q_s \leq 2(\rho_s-1)d + \frac{\rho_s^2\epsilon_s^2}{\gamma_s^2}.
    \]
    By the parameter choice and $\epsilon_s = (1+\sqrt{\alpha})\max\cbr{\beta_s, \beta_{s+1}}$,
    \[
        Q_s \leq \left(\sqrt{\alpha}-\frac{1}{\sqrt{\alpha}}\right)^2.
    \]

    Thus, for each $i \in [d]$,
    \[
        \frac{x_{s+1,i}}{x_{s,i}}
        +
        \frac{x_{s,i}}{x_{s+1,i}}
        -2
        =
        \frac{(x_{s+1,i}-x_{s,i})^2}{x_{s+1,i}x_{s,i}}
        \leq
        Q_s
        \leq
        \alpha+\frac{1}{\alpha}-2.
    \]
    Since $h(u)\triangleq u+u^{-1}-2$ is increasing for $u\ge 1$, taking $u=\frac{x_{s+1,i}}{x_{s,i}}$ implies that either $u\leq 1$ or $u\leq \alpha$ whenever $u>0$. In either case, we have
    \[
        x_{s+1,i}\leq \alpha x_{s,i},
    \]
    as $\alpha \ge d \ge 1$.
    This completes the induction.
\end{proof}

\begin{proof}[Proof of {\pref{lem:mult-stability}}]
    We show that the choices of $\gamma_s$ and $\beta_s$ in \pref{prop:falcon-convergence} satisfy the conditions in \pref{lem:mult-stability-induction} with $\alpha = 16d$. By definition of $\gamma_s$, we know that
$
    \rho_s
    \triangleq
    \frac{\gamma_s}{\gamma_{s+1}}
    =
    \sqrt{
    2\cdot
    \frac{\log(8ds^2/\delta)}
    {\log(8d(s+1)^2/\delta)}
    }\leq \sqrt{2}.
$
For $s\geq 2$,
\begin{align*}
    \log\frac{8d(s+1)^2}{\delta}
    &=
    \log\frac{8ds^2}{\delta}
    +
    2\log\left(1+\frac{1}{s}\right) \leq
    2\log\frac{8ds^2}{\delta},
\end{align*}
since
$
    2\log\left(1+\frac{1}{s}\right)
    \leq
    2\log\frac32
    \leq
    \log\frac{8ds^2}{\delta}.
$
Thus, $1\leq \rho_s\leq \sqrt{2}$. This also implies $\beta_s\geq \beta_{s+1}$, and hence $\max\{\beta_s,\beta_{s+1}\}=\beta_s$.

We now verify the first condition in \pref{lem:mult-stability-induction}. Using
$\gamma_s=5\beta_s$, $\rho_s\leq\sqrt{2}$, and $\alpha=16d$, we have
\begin{align*}
    &2(\rho_s-1)d
    +
    \frac{
    \rho_s^2(1+\sqrt{\alpha})^2\max\{\beta_s,\beta_{s+1}\}^2 
    }{\gamma_s^2} \leq
    2(\sqrt{2}-1)d
    +
    \frac{
    2(1+4\sqrt{d})^2\beta_s^2 
    }{(5\beta_s)^2}\leq
    3d.
\end{align*}
On the other hand,
$
    \left(\sqrt{\alpha}-\frac{1}{\sqrt{\alpha}}\right)^2
    =
    \left(\sqrt{16d}-\frac{1}{\sqrt{16d}}\right)^2
    =
    16d-2+\frac{1}{16d}
    \geq
    14d.
$
Therefore, the conditions in \pref{lem:mult-stability-induction} holds with $\alpha = 16d$, implying that for all $s \ge 2$,
$
    x_{s,i}\leq 16d\,x_{s-1,i}, \forall i\in[d].
$
\end{proof}

Finally, we prove the last-iterate convergence guarantee.
\begin{proof}[Proof of {\pref{prop:falcon-convergence}}]
We prove the claim on the event $\calE_1$, which holds with probability at least
$1-\delta$ by \pref{lem:estimation-error}. Since the length of epoch $s$ is $2^{s-1}$, equivalently, we prove that $\dgap(x_s,y_s)=\order(d2^{-s/2}\sqrt{\log(ds^2/\delta)})$. Using \pref{lem:estimated-reg-bound} and \pref{lem:mult-stability}, the proof follows from the discussion of the single-agent case in \pref{sec:alg}.
\end{proof}


\section{Omitted Proofs for Games}\label{app:falcon_game}

In this section, we show the omitted proofs in \pref{sec:alg} for the general game case. We start with the lemma showing the concentration between $\wh{A}_{s,ij}$ and the true game matrix $A_{ij}$.

\begin{definition}\label{def:closeGame}
    Define $\calE_2$ to be the event: $\calE_2\triangleq\left\{s\geq 2, \forall i,j\in[d],\left|\wh{A}_{s,ij}-A_{ij}\right|\leq \frac{\beta_s}{\sqrt{x_{s-1,i}y_{s-1,j}}}\right\}$ where $\beta_s = \sqrt{\frac{16\log(8d^2s^2/\delta)}{2^{s-2}}}$.
\end{definition}

\begin{restatable}[Estimation error]{lemma}{closeGame}\label{lem:closeGame}
    \pref{alg:game_falcon} guarantees that $\calE_2$ holds with probability at least $1-\delta$.
\end{restatable}

\begin{proof}
    Following a similar analysis to \pref{lem:estimation-error}, we obtain a counterpart for \pref{eqn:concentration} in the game case: given a pair of $(i,j)\in[d]\times[d]$, we know that
    \begin{align*}
        \Pr\left[\abs{\wh{A}_{s,ij}-A_{ij}} > \sqrt{\frac{16\log(4d^2/\xi_s)}{2^{s-2}x_{s-1,i}}}\right] \le e^{-2\log(4d^2/\xi_s)}+2e^{-16\log(4d^2/\xi_s)} \le \frac{\xi_s}{d^2}.
    \end{align*}
    Picking $\xi_s=\frac{\delta}{2s^2}$ and taking a union bound over all $(i,j)$ pairs and $s=1,2,\dots$ complete the proof.
\end{proof}

We then prove \pref{lem:dgapEstGame}, which is the analogue of \pref{lem:estimated-reg-bound} in the single player case. For clarity, we restate the lemma as follows.
\dgapEstGame*
\begin{proof}
    Since $\Phi_s(x,y)$ is strongly convex in $x$ and strongly concave in $y$, $(x_s,y_s)$ is the unique saddle point of $\Phi_s(x,y)$. Therefore we know that for any $x\in\Delta_d$ and $y\in\Delta_d$, $\Phi_s(x_s,y)\leq \Phi_s(x_s,y_s)\leq \Phi_s(x,y_s)$, meaning that 
    \begin{align*}
        x_s = \argmin_{x\in\Delta_d}\Phi_s(x,y_s)=\argmin_{x\in\Delta_d}\left\{x^\top \wh{A}_sy_s + \gamma_s\sum_{i=1}^d\log \frac{1}{x_i}\right\},\\
        y_s = \argmin_{y\in\Delta_d}-\Phi_s(x_s,y)=\argmin_{y\in\Delta_d}\left\{-x_s^\top \wh{A}_sy + \gamma_s\sum_{i=1}^d\log \frac{1}{y_i}\right\}.
    \end{align*}
    With an abuse of notation, for $v\in \R^d$, we denote $\frac{1}{v}\in\R^d$ with the $i$-th coordinate $\frac{1}{v_i}$. According to first-order optimality, we have for any $x\in\Delta_d$ and $y\in\Delta_d$,
    \begin{align*}
        \inner{\wh{A}_sy_s-\frac{\gamma_s}{x_{s}},x-x_s}\geq 0,\\
        \inner{-\wh{A}_s^\top x_s-\frac{\gamma_s}{y_{s}},y-y_s}\geq 0.
    \end{align*}
    Picking $x=e_i$ and $y=e_j$ and summing up the above two inequalities lead to the conclusion.
\end{proof}

Next, we aim to prove \pref{lem:stabGame_concise}, showing the multiplicative stability of \pref{alg:game_falcon} between consecutive epochs. In fact, we will prove the following more general \pref{lem:stabGame} and \pref{lem:stabGame_concise} will serve as a corollary by picking the parameter $\gamma_s$ as stated in \pref{thm:dgapGame}.

\begin{restatable}[Multiplicative stability]{lemma}{stabGame}
\label{lem:stabGame}

    Suppose that $\calE_2$ holds and define $\rho_s=\gamma_s/\gamma_{s+1}$. Further suppose that the sequence
    $\{\gamma_s\}_{s\geq 1}$ in \pref{alg:game_falcon} satisfies (1) for all $s\geq 2$,
    \(
    4(\rho_s-1)d
    +
    \frac{
    2\rho_s^2(1+\alpha)^2\max\{\beta_s,\beta_{s+1}\}^2 d
    }{\gamma_s^2}
    \leq
    \left(\sqrt{\alpha}-\frac{1}{\sqrt{\alpha}}\right)^2
    \)
    for certain $\alpha\geq d$. (2) $\rho_s\in[1,\sqrt{2}]$. Then \pref{alg:game_falcon}
    guarantees that, for all $s\geq 2$,
    $
        x_{s,i}\leq \alpha x_{s-1,i},
        y_{s,j}\leq \alpha y_{s-1,j},\forall i,j\in[d].
    $
\end{restatable}
\begin{proof}
    Let $h(u)=u+u^{-1}-2$. We prove the result by induction. The first non-trivial case $s=2$ follows from the initialization $x_1=y_1=\one/d$, since
    \[
        x_{2,i}\leq 1\leq d x_{1,i}\leq \alpha x_{1,i},
        \qquad
        y_{2,j}\leq 1\leq d y_{1,j}\leq \alpha y_{1,j},
    \]
    where we used $\alpha\geq d$. Suppose that for some $s\geq 2$, $ x_{s,i}\leq \alpha x_{s-1,i}, y_{s,j}\leq \alpha y_{s-1,j},
         \forall i,j\in[d].
    $
    We prove the same conclusion for $s+1$. Define
    \[
        \bar\beta_s\triangleq \max\{\beta_s,\beta_{s+1}\},
        \qquad
        \varepsilon_s\triangleq (1+\alpha)\bar\beta_s.
    \]
    By $\calE_2$ and the induction hypothesis, for every $i,j\in[d]$,
    \begin{align}
        \left|\wh A_{s,ij}-\wh A_{s+1,ij}\right|
        &\leq
        \left|\wh A_{s,ij}-A_{ij}\right|
        +
        \left|\wh A_{s+1,ij}-A_{ij}\right| \notag\\
        &\leq
        \frac{\beta_s}{\sqrt{x_{s-1,i}y_{s-1,j}}}
        +
        \frac{\beta_{s+1}}{\sqrt{x_{s,i}y_{s,j}}} \notag\\
        &\leq
        \frac{\alpha\beta_s+\beta_{s+1}}{\sqrt{x_{s,i}y_{s,j}}}
        \leq
        \frac{\varepsilon_s}{\sqrt{x_{s,i}y_{s,j}}}.
        \label{eqn:A-diff-stab}
    \end{align}

    Let
    \[
        Q_s\triangleq
        \sum_{i=1}^d
        \frac{(x_{s+1,i}-x_{s,i})^2}{x_{s+1,i}x_{s,i}},
        \qquad
        R_s\triangleq
        \sum_{j=1}^d
        \frac{(y_{s+1,j}-y_{s,j})^2}{y_{s+1,j}y_{s,j}}.
    \]

    Applying first-order optimality at phases $s$ and $s+1$, we know that for any $x\in\Delta_d$,
    \begin{align}
        \inner{\wh{A}_sy_s-\frac{\gamma_s}{x_s},x-x_s}\geq 0,
        \label{eqn:first-order-x-r}\\
        \inner{\wh{A}_{s+1}y_{s+1}-\frac{\gamma_{s+1}}{x_{s+1}},x-x_{s+1}}\geq 0.
        \label{eqn:first-order-x-r-plus-one}
    \end{align}
    Picking $x=x_{s+1}$ in \pref{eqn:first-order-x-r} and $x=x_s$ in \pref{eqn:first-order-x-r-plus-one}, and then summing the two inequalities, gives
    \[
        Q_s
        \leq
        \inner{x_{s+1}-x_s,
        \frac{\wh A_s y_s}{\gamma_s}
        -
        \frac{\wh A_{s+1}y_{s+1}}{\gamma_{s+1}} }.
    \]
    Let $\rho_s\triangleq \gamma_s/\gamma_{s+1}$. A further decomposition shows that
    \begin{align}
        Q_s
        &\leq 
        \frac{\rho_s}{\gamma_s}
        \inner{x_{s+1}-x_s,\wh A_{s+1}(y_s-y_{s+1})}
        +
        \frac{1-\rho_s}{\gamma_s}
        \inner{x_{s+1}-x_s,\wh A_s y_s} \notag\\
        &\qquad
        +
        \frac{\rho_s}{\gamma_s}
        \inner{x_{s+1}-x_s,(\wh A_s-\wh A_{s+1})y_s}.
        \label{eqn:x-stability-decomp}
    \end{align}

    Similarly, first-order optimality of the $y$-updates gives
    \begin{align}
        R_s
        &\leq
        \frac{\rho_s}{\gamma_s}
        \inner{y_{s+1}-y_s,\wh A_{s+1}^\top(x_{s+1}-x_s)}
        +
        \frac{\rho_s-1}{\gamma_s}
        \inner{y_{s+1}-y_s,\wh A_s^\top x_s} \notag\\
        &\qquad
        +
        \frac{\rho_s}{\gamma_s}
        \inner{y_{s+1}-y_s,(\wh A_{s+1}-\wh A_s)^\top x_s}.
        \label{eqn:y-stability-decomp}
    \end{align}

    We now add \pref{eqn:x-stability-decomp} and \pref{eqn:y-stability-decomp}. The first terms cancel because
    \[
        \inner{x_{s+1}-x_s,\wh A_{s+1}(y_s-y_{s+1})}
        +
        \inner{y_{s+1}-y_s,\wh A_{s+1}^\top(x_{s+1}-x_s)}
        =0.
    \]

    We next control the two terms caused by the change of $\gamma_s$. By \pref{eqn:first-order-x-r} with $x=x_{s+1}$,
    \[
        \inner{x_{s+1}-x_s,\frac{\wh A_s y_s}{\gamma_s}}
        \geq
        \inner{x_{s+1}-x_s,\frac{1}{x_s}}.
    \]
    Since $\rho_s\geq 1$, this implies
    \begin{align*}
        \frac{1-\rho_s}{\gamma_s}
        \inner{x_{s+1}-x_s,\wh A_s y_s}
        &\leq
        (1-\rho_s)
        \inner{x_{s+1}-x_s,\frac{1}{x_s}} \\
        &=
        (\rho_s-1)
        \left(
        d-\inner{\frac{1}{x_s},x_{s+1}}
        \right)
        \leq
        (\rho_s-1)d.
    \end{align*}
    Similarly, first-order optimality of the $y$-update at phase $s$ gives
    \[
        \inner{-\wh A_s^\top x_s-\frac{\gamma_s}{y_s},y-y_s}\geq 0,
        \qquad \forall y\in\Delta_d.
    \]
    Taking $y=y_{s+1}$ yields
    \[
        \inner{y_{s+1}-y_s,\frac{\wh A_s^\top x_s}{\gamma_s}}
        \leq
        -\inner{y_{s+1}-y_s,\frac{1}{y_s}}.
    \]
    Therefore,
    \begin{align*}
        \frac{\rho_s-1}{\gamma_s}
        \inner{y_{s+1}-y_s,\wh A_s^\top x_s}
        &\leq
        -(\rho_s-1)
        \inner{y_{s+1}-y_s,\frac{1}{y_s}} \\
        &=
        (\rho_s-1)
        \left(
        d-\inner{\frac{1}{y_s},y_{s+1}}
        \right)
        \leq
        (\rho_s-1)d.
    \end{align*}

    It remains to bound the estimator-difference terms. By \pref{eqn:A-diff-stab},
    \begin{align*}
        \left|
        \inner{x_{s+1}-x_s,(\wh A_s-\wh A_{s+1})y_s}
        \right|
        &\leq
        \sum_{i=1}^d |x_{s+1,i}-x_{s,i}|
        \sum_{j=1}^d
        \left|\wh A_{s,ij}-\wh A_{s+1,ij}\right|y_{s,j} \\
        &\leq
        \varepsilon_s
        \sum_{i=1}^d
        \frac{|x_{s+1,i}-x_{s,i}|}{\sqrt{x_{s,i}}}
        \sum_{j=1}^d \sqrt{y_{s,j}} \\
        &\leq
        \varepsilon_s\sqrt d
        \sum_{i=1}^d
        \frac{|x_{s+1,i}-x_{s,i}|}{\sqrt{x_{s,i}}} \\
        &=
        \varepsilon_s\sqrt d
        \sum_{i=1}^d
        \frac{|x_{s+1,i}-x_{s,i}|}{\sqrt{x_{s+1,i}x_{s,i}}}
        \sqrt{x_{s+1,i}} \\
        &\leq
        \varepsilon_s\sqrt{dQ_s}.
    \end{align*}
    Similarly,
    \[
        \left|
        \inner{y_{s+1}-y_s,(\wh A_{s+1}-\wh A_s)^\top x_s}
        \right|
        \leq
        \varepsilon_s\sqrt{dR_s}.
    \]

    Combining the above bounds gives
    \[
        Q_s+R_s
        \leq
        2(\rho_s-1)d
        +
        \frac{\rho_s\varepsilon_s\sqrt d}{\gamma_s}
        \left(\sqrt{Q_s}+\sqrt{R_s}\right).
    \]
    Let $S_s\triangleq Q_s+R_s$. Since
    \[
        \sqrt{Q_s}+\sqrt{R_s}\leq \sqrt{2S_s},
    \]
    we have
    \[
        S_s
        \leq
        2(\rho_s-1)d
        +
        \frac{\rho_s\varepsilon_s\sqrt d}{\gamma_s}\sqrt{2S_s}.
    \]
    Using $a\sqrt{2S_s}\leq S_s/2+a^2$ with
    \[
        a=\frac{\rho_s\varepsilon_s\sqrt d}{\gamma_s},
    \]
    we obtain
    \[
        S_s
        \leq
        4(\rho_s-1)d
        +
        \frac{2\rho_s^2\varepsilon_s^2 d}{\gamma_s^2}.
    \]
    Recalling that $\varepsilon_s=(1+\alpha)\max\{\beta_s,\beta_{s+1}\}$, the assumed parameter condition implies
    \[
        Q_s+R_s=S_s
        \leq
        \left(\sqrt{\alpha}-\frac{1}{\sqrt{\alpha}}\right)^2
        =
        h(\alpha).
    \]

    Hence, for every coordinate $i$,
    \[
        \frac{x_{s+1,i}}{x_{s,i}}
        +
        \frac{x_{s,i}}{x_{s+1,i}}
        -2
        =
        \frac{(x_{s+1,i}-x_{s,i})^2}{x_{s+1,i}x_{s,i}}
        \leq
        Q_s
        \leq
        h(\alpha).
    \]
    Since $u+u^{-1}-2\leq h(\alpha)$ implies $u\leq \alpha$ for every $u>0$, we get
    \[
        x_{s+1,i}\leq \alpha x_{s,i}.
    \]
    The proof for $y$ is identical, using $R_s\leq h(\alpha)$, and gives
    \[
        y_{s+1,j}\leq \alpha y_{s,j}.
    \]
    Thus the induction closes.
\end{proof}

\pref{lem:stabGame_concise} is then proven by verifying the $\gamma_s$ choice used in \pref{thm:dgapGame} satisfies the condition of \pref{lem:stabGame}.
\begin{proof}[Proof of {\pref{lem:stabGame_concise}}]
    We verify that the choice of $\gamma_s$ and $\beta_s$ satisfy the condition in \pref{lem:stabGame} with $\alpha = 16d$. By definition of $\gamma_s$, we know that
$
    \rho_s
    \triangleq
    \frac{\gamma_s}{\gamma_{s+1}}
    =
    \frac{\beta_s}{\beta_{s+1}}
    =
    \sqrt{
    2\cdot
    \frac{\log(8d^2s^2/\delta)}
    {\log(8d^2(s+1)^2/\delta)}
    }\leq \sqrt{2}.
$
Moreover, for $s\geq 2$,
\begin{align*}
    \log\frac{8d^2(s+1)^2}{\delta}
    &=
    \log\frac{8d^2s^2}{\delta}
    +
    2\log\left(1+\frac{1}{s}\right) \leq
    2\log\frac{8d^2s^2}{\delta},
\end{align*}
where the inequality follows from
$
    2\log\left(1+\frac{1}{s}\right)
    \leq
    2\log\frac32
    \leq
    \log\frac{8d^2s^2}{\delta}.
$
Therefore, we have $1\leq \rho_s\leq \sqrt{2}$. In addition, it also means that $\beta_s\geq \beta_{s+1}$, and hence $\max\{\beta_s,\beta_{s+1}\}=\beta_s$.

Now we check the first condition in \pref{lem:stabGame}. Using
$\gamma_s=16d\beta_s$, $\rho_s\leq\sqrt{2}$, and $\alpha=16d$, we have
\begin{align*}
    &4(\rho_s-1)d
    +
    \frac{
    2\rho_s^2(1+\alpha)^2\max\{\beta_s,\beta_{s+1}\}^2 d
    }{\gamma_s^2} \leq
    4(\sqrt{2}-1)d
    +
    \frac{
    4(1+16d)^2\beta_s^2 d
    }{(16d\beta_s)^2}\leq
    7d.
\end{align*}
On the other hand,
$
    \left(\sqrt{\alpha}-\frac{1}{\sqrt{\alpha}}\right)^2
    =
    \left(\sqrt{16d}-\frac{1}{\sqrt{16d}}\right)^2
    =
    16d-2+\frac{1}{16d}
    \geq
    14d.
$
Therefore the conditions in \pref{lem:stabGame} hold with $\alpha=16d$. Hence, for all $s\geq 2$,
$
    x_{s,i}\leq 16d\,x_{s-1,i}, y_{s,j}\leq 16d\,y_{s-1,j}, \forall i,j\in[d].
$
\end{proof}

With all the above lemmas, we are ready to prove our main result \pref{thm:dgapGame}.

\begin{proof}[Proof of {\pref{thm:dgapGame}}]
We prove the claim on the event $\calE_2$, which holds with probability at least
$1-\delta$ by \pref{lem:closeGame}. Since the length of epoch $s$ is $2^{s-1}$, equivalently, we prove that $\dgap(x_s,y_s)=\order(d^22^{-s/2}\sqrt{\log(d^2s^2/\delta)})$. 
To proceed with the proof, fix any epoch $s\geq 2$ and any pair
$i,j\in[d]$. We first decompose
\begin{align}\label{eqn:dualDecompose}
    x_s^\top A e_j - e_i^\top A y_s
    &=
    x_s^\top \wh A_s e_j - e_i^\top \wh A_s y_s
    + x_s^\top (A-\wh A_s)e_j
    + e_i^\top(\wh A_s-A)y_s .
\end{align}
By \pref{lem:dgapEstGame}, we know that
\begin{align*}
    x_s^\top \wh A_s e_j - e_i^\top \wh A_s y_s
    \leq
    2\gamma_s d
    - \frac{\gamma_s}{x_{s,i}}
    - \frac{\gamma_s}{y_{s,j}} .
\end{align*}
It remains to control the last two estimation-error terms in \pref{eqn:dualDecompose}. For the first one, using $\calE_2$ gives
\begin{align*}
    x_s^\top (A-\wh A_s)e_j
    &\leq
    \sum_{k=1}^d x_{s,k}
    \left|A_{kj}-\wh A_{s,kj}\right| \leq
    \beta_s
    \sum_{k=1}^d
    \frac{x_{s,k}}{\sqrt{x_{s-1,k}y_{s-1,j}}}.
\end{align*}
By \pref{lem:stabGame}, $x_{s,k}\leq 16d\,x_{s-1,k}$ and
$y_{s,j}\leq 16d\,y_{s-1,j}$. Hence
\begin{align*}
    \frac{x_{s,k}}{\sqrt{x_{s-1,k}y_{s-1,j}}}
    &\leq
    16d\cdot \frac{\sqrt{x_{s,k}}}{\sqrt{y_{s,j}}}.
\end{align*}
Therefore, by Cauchy-Schwarz,
\begin{align*}
    x_s^\top (A-\wh A_s)e_j
    \leq
    16d\,\beta_s
    \frac{\sum_{k=1}^d \sqrt{x_{s,k}}}{\sqrt{y_{s,j}}} \leq
    16d\,\beta_s
    \frac{\sqrt{d}}{\sqrt{y_{s,j}}}
    =
    \frac{16d^{3/2}\beta_s}{\sqrt{y_{s,j}}}.
\end{align*}
Similarly, for the other term, we have
\begin{align*}
    e_i^\top(\wh A_s-A)y_s
    \leq
    \sum_{\ell=1}^d y_{s,\ell}
    \left|\wh A_{s,i\ell}-A_{i\ell}\right|  \leq
    \beta_s
    \sum_{\ell=1}^d
    \frac{y_{s,\ell}}{\sqrt{x_{s-1,i}y_{s-1,\ell}}} \leq
    \frac{16d^{3/2}\beta_s}{\sqrt{x_{s,i}}}.
\end{align*}
Combining the three bounds, for every $i,j\in[d]$,
\begin{align*}
    x_s^\top A e_j - e_i^\top A y_s
    \leq
    2\gamma_s d
    - \frac{\gamma_s}{x_{s,i}}
    - \frac{\gamma_s}{y_{s,j}}
    + \frac{16d^{3/2}\beta_s}{\sqrt{x_{s,i}}}
    + \frac{16d^{3/2}\beta_s}{\sqrt{y_{s,j}}}\leq 2\gamma_s d
    + \frac{128d^3\beta_s^2}{\gamma_s},
\end{align*}
where the last inequality is due to AM-GM inequality. We finish the proof by plugging in $\gamma_s=16d\beta_s$ and $\beta_s = \sqrt{\frac{16\log(8d^2s^2/\delta)}{2^{s-2}}}$ and noticing that $\dgap(x_s,y_s)=\max_{i,j\in[d]}\left\{x_s^\top A e_j - e_i^\top A y_s\right\}$.
\end{proof}

\section{Additional Omitted Details for \pref{sec:obstacle}}\label{app:obstacle}

\subsection{Last-Iterate Convergence for Empirical NE with Uniform Exploration}\label{app:NEUniform}
In this section, we present one of our baseline algorithms that mixes the empirically best strategy with uniform exploration. \pref{alg:ne-estimate} also operates in an epoch-wise manner. 

\begin{algorithm}[h]
\caption{Empirical NE with uniform exploration}\label{alg:ne-estimate}
Initialize game estimator $\wh{A}_1$ to be the all-zero $d$ by $d$ matrix.

\For{$s = 1, 2, \dots $}{
    Compute $\alpha_s = d^{1/2}2^{-(s-1)/4}$.
    
    For row player: compute $\wh{x}_s\in \argmin_{x\in \Delta_d}\max_{y\in \Delta_d} x^\top\wh{A}_sy$ and $x_s=(1-\alpha_s)\wh{x}_s+\alpha_s\frac{\one}{d}$.
    

    For column player: compute $\wh{y}_s\in \argmax_{y\in \Delta_d}\min_{x\in \Delta_d} x^\top\wh{A}_sy$ and $y_s=(1-\alpha_s)\wh{y}_s+\alpha_s\frac{\one}{d}$.

    Initialize counters: $n_{s,ij}=0$ for all $i,j\in[d]$.
    
    \For{$t=2^{s-1},\dots,2^{s}-1$}{
        Row player samples $i_t\sim x_s$ and column player samples $j_t\sim y_s$.
        
        Both players observe $(i_t, j_t)$ and $r_t$ with $\E[r_t]=A_{i_tj_t}$.
        
        Increment counter $n_{s,i_tj_t}\leftarrow n_{s,i_tj_t}+1$.
    }

    Compute game estimator $\wh{A}_{s+1,ij}=\frac{\mathbbm{1}\cbr{n_{s,ij}\neq 0}}{n_{s,ij}}\sum_{\tau=2^{s-1}}^{2^{s}-1}r_\tau\cdot \mathbbm{1}\{i_\tau=i,j_\tau=j\}$ for all $i,j\in [d]$.
}
\end{algorithm}

\begin{proposition}\label{prop:ne-estimate}
    For any fixed $\delta \in (0,1)$,
    \pref{alg:ne-estimate} guarantees that with probability at least $1-\delta$, $\dgap(x_t,y_t)=\order\rbr{\sqrt{d\log(d\log t/\delta)}t^{-\frac{1}{4}}}$ for all $t\geq 1$, where $(x_t,y_t)$ denotes the strategy pair played in round $t$.
\end{proposition}

\begin{proof}
    We prove the claim on the event $\calE_2$, which holds with probability at least $1-\delta$ by \pref{lem:closeGame}. Since epoch $s$ has length $2^{s-1}$, it suffices to show that $\dgap(x_s,y_s)=\order(2^{-s/4}\sqrt{d\log(d^2s^2/\delta)})$. Let $(\wh{x}_s,\wh{y}_s)$ denote a Nash Equilibrium of $\wh{A}_s$.
    Fix any epoch $s\geq 2$. We begin by controlling the duality gap with respect to the empirical estimate $\wh{A}_s$. Note that $\wh A_{s,ij}\in[-1,1]$ holds for all $i,j\in[d]$. From the sampling strategy of \pref{alg:ne-estimate}, we have $\|x_s-\wh x_s\|_1\le 2\alpha_s$ and $\|y_s-\wh y_s\|_1\le 2\alpha_s$.
    Therefore, for all pairs $i,j\in[d]$,
    \begin{align*}
        &x_s^\top \wh A_s e_j \le \wh x_s^\top \wh A_s e_j + 2\alpha_s
        \le \max_{y\in\Delta_d}\wh x_s^\top \wh A_s y
        + 2\alpha_s = \wh x_s^\top \wh A_s \wh y_s
        + 2\alpha_s,\\
        &e_i^\top \wh A_s y_s \ge e_i^\top \wh A_s \wh y_s - 2\alpha_s \ge \min_{x\in\Delta_d}x^\top \wh A_s \wh y_s - 2\alpha_s = \wh x_s^\top \wh A_s \wh y_s - 2\alpha_s.
    \end{align*}
    Thus, for all $i,j\in[d]$, we have $x_s^\top \wh A_s e_j - e_i^\top \wh A_s y_s \le 4\alpha_s$.

    It remains to bound the last two estimation error terms in \pref{eqn:dualDecompose}. For the first term, \pref{lem:closeGame} implies
    \begin{align*}
        x_s^\top (A-\wh A_s)e_j
        &\leq
        \sum_{k=1}^d x_{s,k}
        \left|A_{kj}-\wh A_{s,kj}\right| \leq
        \beta_s
        \sum_{k=1}^d
        \frac{x_{s,k}}{\sqrt{x_{s-1,k}y_{s-1,j}}}\le \frac{\beta_sd}{\alpha_{s-1}} = \frac{\beta_sd}{2\alpha_s},
    \end{align*}
    where the last inequality follows from the fact that $x_{s-1,k}\ge \frac{\alpha_{s-1}}{d}$ and $y_{s-1,j}\ge\frac{\alpha_{s-1}}{d}$ for all $k,j\in[d]$, due to mixing with the uniform distribution.

    Similarly, for the other term
    \begin{align*}
        e_i^\top(\wh A_s-A)y_s
        \leq
        \sum_{\ell=1}^d y_{s,\ell}
        \left|\wh A_{s,i\ell}-A_{i\ell}\right|  \leq
        \beta_s
        \sum_{\ell=1}^d
        \frac{y_{s,\ell}}{\sqrt{x_{s-1,i}y_{s-1,\ell}}} \leq
        \frac{\beta_sd}{\alpha_{s-1}}= \frac{\beta_sd}{2\alpha_s}.
    \end{align*}

    Combining the three bounds, for every $i,j\in[d]$, we get
    \begin{align*}
        x_s^\top A e_j - e_i^\top A y_s
        \leq 4\alpha_s + \frac{\beta_sd}{\alpha_s}.
    \end{align*}
    We have $\beta_s = \sqrt{\frac{16\log(8d^2s^2/\delta)}{2^{s-2}}}$. Setting $\alpha_s$ optimally as $\sqrt{\beta_s d}$ yields $\dgap(x_s,y_s)=\max_{i,j\in[d]}\left\{x_s^\top A e_j - e_i^\top A y_s\right\} = \order(2^{-s/4}\sqrt{d\log(d^2s^2/\delta)})$.
\end{proof}

\subsection{Discussion on Low-Regret-Low-Variance Property in Games}\label{app:low-reg-low-var}
In this subsection, we explain why a low-duality-gap-low-variance property, analogous to \pref{lem:low-reg-low-var}, is insufficient to establish last-iterate convergence in the game setting. We begin by presenting the following lemma, the game counterpart to \pref{lem:low-reg-low-var}, which follows directly from \pref{lem:dgapEstGame}.

\begin{lemma}\label{lem:low-reg-low-var-game}
    Define $\wh{\dgap}_s(x,y)=\max_{j\in[d]}x^\top\wh{A}_s e_j - \min_{i\in[d]}e_i^\top\wh{A}_sy$ for all $s\geq 1$. For every epoch $s\geq 1$, \pref{alg:game_falcon} satisfies:
    \begin{align}
        \wh{\dgap}_s(x_s,y_s)&\leq 2d\gamma_s,\\
        \frac{1}{x_{s,i}} + \frac{1}{y_{s,j}} &\leq 2d + \frac{\wh{\dgap}_s(e_i,e_j)}{\gamma_s},~\forall i,j\in[d]. \label{eq:low-variance_game}
    \end{align}
\end{lemma}
\begin{proof}
    By \pref{lem:dgapEstGame}, for all $i,j\in[d]$, we have:
    \begin{align*}
        x_s^\top \wh{A}_{s}e_j - e_i^\top \wh{A}_{s}y_s \leq 2d\gamma_s  - \frac{\gamma_s}{x_{s,i}} - \frac{\gamma_s}{y_{s,j}} \leq 2d\gamma_s.
    \end{align*}
    Furthermore, rearranging the terms in \pref{eqn:dgapEstGame} yields:
    \begin{align*}
        \frac{1}{x_{s,i}} + \frac{1}{y_{s,j}} \leq 2d - \frac{1}{\gamma_s} \left(x_s^\top \wh{A}_{s}e_j - e_i^\top \wh{A}_{s}y_s\right) \leq 2d + \frac{\wh{\dgap}_s(e_i,e_j)}{\gamma_s}.
    \end{align*}
    This completes the proof.
\end{proof}

Unlike the single-agent case, however, we cannot directly bound the true duality gap $\dgap(x_s,y_s)$ even if the following multiplicative closeness analogous to the single-player case holds due to \pref{eq:low-variance_game}:
\begin{align*}
    \wh{\dgap}_s(e_i,e_j)&\leq 2\dgap(e_i,e_j) + \otil(d\gamma_s),\\
    {\dgap}(e_i,e_j)&\leq 2\wh{\dgap}_s(e_i,e_j) + \otil(d\gamma_s).
\end{align*}

This is because of the following crucial distinction: In the single-agent case, expectations translate linearly (i.e., $\E_{i\sim x_s}[\wh{\Reg}_s(e_i)]=\wh{\Reg}_s(x_s)$ and $\E_{i\sim x_s}[{\Reg}(e_i)]={\Reg}(x_s)$), which allows multiplicative closeness for pure actions to neatly imply closeness for mixed strategies. In the game setting, this linear translation breaks down. Instead, we are restricted to Jensen's inequality ($\E_{i\sim x_s,j\sim y_s}[\dgap(e_i,e_j)]\geq \dgap(x_s,y_s)$), which fundamentally prevents the application of a similar argument.